\documentclass[11pt]{article}

\usepackage{amsmath,amsfonts,amsthm,amssymb,color}

\usepackage{fullpage}

\newtheorem{theorem}{Theorem}[section]
\newtheorem{corollary}[theorem]{Corollary}
\newtheorem{lemma}[theorem]{Lemma}
\newtheorem{proposition}[theorem]{Proposition}
\newtheorem{definition}[theorem]{Definition}

\newtheorem{remark}[theorem]{Remark}

\usepackage{hyperref}
\usepackage[ruled,vlined]{algorithm2e}
\usepackage{graphicx}
\usepackage{xspace}
\usepackage{ifthen}

\newcommand{\eps}{\ensuremath{\epsilon}\xspace}
\renewcommand{\tilde}{\widetilde}
\renewcommand{\hat}{\widehat}

\newcommand{\R}{\mathbb{R}}

\DeclareMathOperator{\tr}{tr}
\newcommand{\norm}[1]{\left\|#1\right\|}
\newcommand{\normone}[1]{\norm{#1}_1}
\newcommand{\normtwo}[1]{\norm{#1}_2}

\providecommand{\expect}[2]{\ensuremath{\ifthenelse{\equal{#1}{}}{\mathbb{E}}{\mathbb{E}_{#1}}\!\left[#2\right]}\xspace}
\providecommand{\prob}[2]{\ensuremath{\ifthenelse{\equal{#1}{}}{\Pr}{\Pr_{#1}}\!\left[#2\right]}\xspace}

\newcommand{\mini}[1]{\mbox{minimize} & {#1} &\\}
\newcommand{\maxi}[1]{\mbox{maximize} & {#1 } & \\}
\newcommand{\st}{\mbox{subject to} }
\newcommand{\con}[1]{&#1 & \\}

\newenvironment{lp}{\begin{equation}  \begin{array}{lll}}{\end{array}\end{equation}}
\newenvironment{lp*}{\begin{equation*}  \begin{array}{lll}}{\end{array}\end{equation*}}

\newcommand{\inner}[1]{\langle #1\rangle}

\usepackage{bm}

\newcommand{\NN}{\mathcal{N}}
\newcommand{\mus}{\mu^\star}
\newcommand{\abs}[1]{\left|{#1}\right|}

\DeclareMathOperator{\poly}{poly}
\newcommand{\OPT}{\mathrm{OPT}\xspace}
\newcommand{\ALG}{\mathrm{ALG}\xspace}
\newcommand{\mone}{\mathbf{1}}

\def\colorful{1}

\ifnum\colorful=1

\else

\fi

\title{High-Dimensional Robust Mean Estimation in Nearly-Linear Time}

\author{
Yu Cheng\thanks{Supported in part by NSF grants CCF-1704656, CCF-1527084, CCF-1535972, CCF-1637397, IIS-1447554, and NSF CAREER Award CCF-1750140.}\\
Duke University\\
{\tt yucheng@cs.duke.edu}\\
\and
Ilias Diakonikolas\thanks{Supported by NSF Award CCF-1652862 (CAREER) and a Sloan Research Fellowship.}\\
University of Southern California\\
{\tt diakonik@usc.edu}\\
\and
Rong Ge\thanks{Supported by NSF Award CCF-1704656.}\\
Duke University\\
{\tt rongge@cs.duke.edu}
}

\begin{document}

\maketitle

\date{}

\setcounter{page}{0}

\thispagestyle{empty}

\begin{abstract}
We study the fundamental problem of high-dimensional mean estimation in a robust model 
where a constant fraction of the samples are adversarially corrupted. Recent work gave the first polynomial time algorithms 
for this problem with dimension-independent error guarantees for several families of structured distributions. 

In this work, we give the first {\em nearly-linear time} algorithms for high-dimensional robust mean estimation.
Specifically, we focus on distributions with (i) known covariance and sub-gaussian tails,
and (ii) unknown bounded covariance. Given $N$ samples on $\R^d$, an $\eps$-fraction 
of which may be arbitrarily corrupted, our algorithms run in time $\widetilde{O}(Nd) / \mathrm{poly}(\eps)$ 
and approximate the true mean within
the information-theoretically optimal error, up to constant factors. Previous robust algorithms with comparable 
error guarantees have running times $\widetilde{\Omega}(Nd^2)$, for $\eps = \Omega(1)$. 

Our algorithms rely on a natural family of SDPs parameterized by our current guess $\nu$
for the unknown mean $\mus$. We give a win-win analysis establishing the following: 
either a near-optimal solution to the primal SDP yields a good candidate for $\mus$ ---  
independent of our current guess $\nu$ --- or a near-optimal solution to the dual SDP yields a new guess $\nu'$
whose distance from $\mus$ is smaller by a constant factor. We exploit the special structure
of the corresponding SDPs to show that they are approximately solvable in nearly-linear time.
Our approach is quite general, and we believe it can also be applied to obtain nearly-linear time algorithms for 
other high-dimensional robust learning problems.
\end{abstract}

\newpage


\section{Introduction}
\label{sec:intro}

\subsection{Background}

Consider the following statistical task: Given $N$ independent samples from an unknown mean and 
identity covariance Gaussian distribution $\mathcal{N}(\mus, I)$ 
on $\mathbb{R}^d$, estimate its mean vector $\mus$ within small $\ell_2$-norm. 
It is straightforward to see that the empirical mean --- the average of the samples --- 
has $\ell_2$-error at most $O(\sqrt{d/N})$ from $\mus$ with high probability.
Moreover, this error upper bound is best possible, within constant factors, among all $N$-sample estimators.
That is, in the aforementioned basic setting, there is a sample-optimal mean estimator 
that runs in linear time. 

In this paper, we study the robust (or agnostic) setting when a constant $\eps < 1/2$ fraction of our samples
can be adversarially corrupted. We consider the following model of robust estimation (see, e.g.,~\cite{DKKLMS16})
that generalizes other existing models, including Huber's contamination model~\cite{Huber64}:

\begin{definition} \label{def:adv}
Given $0< \eps < 1/2$ and a family of distributions $\mathcal{D}$ on $\R^d$, 
the \emph{adversary} operates as follows: The algorithm specifies some number of samples $N$, and $N$ samples $X_1, X_2, \ldots, X_N$ are drawn from some (unknown) $D \in \mathcal{D}$.
The adversary is allowed to inspect the samples, removes $\eps N$ of them, 
and replaces them with arbitrary points. This set of $N$ points is then given to the algorithm.
We say that a set of samples is {\em $\eps$-corrupted}
if it is generated by the above process.
\end{definition}


In the context of robust mean estimation studied in this paper, 
the goal is to output a hypothesis vector $\widehat{\mu}$ such that 
$\|\widehat{\mu} - \mus\|_2$ is as small as possible. How do we estimate $\mus$ in this regime? 
A moment's thought reveals that the empirical mean inherently fails in the robust setting: 
even a single corrupted sample can arbitrarily compromise its performance. 
However, one can construct more sophisticated estimators that are provably robust.
The information-theoretically optimal error for robustly estimating the mean 
of $\mathcal{N}(\mus, I)$ is $\Theta(\eps + \sqrt{d/N})$~\cite{Tukey75, Donoho92, CGR15b}. 
That is, when there are enough samples ($N = \Omega(d/\eps^2)$) one can estimate the mean 
to accuracy $\Theta(\eps)$.~\footnote{Under different assumptions on the distribution of the good data, 
the optimal error guarantee may be different as well (see Section~\ref{subsec:results}).} 
However, the standard robust estimators (e.g., Tukey's median~\cite{Tukey75}) require exponential time in the dimension $d$ to compute.
On the other hand, a number of natural approaches (e.g., naive outlier removal, coordinate-wise median, geometric median, etc.) 
can only guarantee error $\Omega(\eps \sqrt{d})$ (see, e.g.,~\cite{DKKLMS16, LaiRV16}), 
even in the infinite sample regime. That is, the performance of these estimators degrades polynomially 
with the dimension $d$, which is clearly unacceptable in high dimensions.

Recent work~\cite{DKKLMS16, LaiRV16} gave the first polynomial time robust estimators for a range of 
high-dimensional statistical tasks, including mean and covariance estimation. Specifically,~\cite{DKKLMS16}
obtained the first robust estimators for the mean with {\em dimension-independent} error guarantees, i.e., whose
error only depends on the fraction of corrupted samples $\epsilon$ but {\em not} on the dimensionality of the data. 
Since the dissemination of ~\cite{DKKLMS16, LaiRV16}, there has been a substantial number of subsequent
works obtaining robust learning algorithms for a variety of unsupervised and supervised high-dimensional models.
(See Section~\ref{ssec:prior} for a summary of related work.)

Although the aforementioned works gave polynomial time robust learning algorithms for several fundamental learning tasks, 
these algorithms are at least a factor $d$ slower than their non-robust counterparts 
(e.g., the sample average for the case of mean estimation), hence 
are significantly slower in high dimensions. It is an important goal to design robust learning algorithms 
with near-optimal sample complexity that are also nearly as efficient as their non-robust counterparts. 
In particular, we propose the following broad question:


\begin{quote}
{\em Can we design (nearly-)sample optimal robust learning algorithms --- with dimension independent error guarantees 
--- that run in {\bf nearly-linear time}?}
\end{quote}

Here by {\em nearly-linear time}, we mean that the runtime is 
proportional to the size of the input, within poly-logarithmic in the input size and $\poly(1/\eps)$ factors. 
In addition to its potential practical implications, we believe that understanding the 
above question is of fundamental theoretical interest as it can
elucidate the effect of the robustness requirement 
on the computational complexity of high-dimensional statistical learning/estimation.

For example, for the prototypical problem of robustly estimating the mean of a high-dimensional distribution, 
previous robust algorithms~\cite{DKKLMS16, LaiRV16, SteinhardtCV18} have runtime at least $\Omega(N d^2)$ for constant
$\eps$. Since the input size is $\Theta (N d)$, we would like to obtain algorithms that run in time $\tilde{O}(N d)/\poly(\eps)$, 
where the $\tilde{O}(\cdot)$ notation hides logarithmic factors in its argument. As the main contribution of this paper,
we obtain such algorithms under different assumptions about 
the distribution of the good data. Our algorithms have optimal sample complexity, 
provide the information-theoretically optimal accuracy, and --- importantly --- run in time $\tilde{O}(N d)/\poly(\eps)$.





\subsection{Our Results}
\label{subsec:results}
Our first algorithmic result handles the setting where the good data distribution
is sub-gaussian with known covariance. Recall that
a distribution $D$ on $\R^d$ with mean $\mus$ is sub-gaussian 
if for any unit vector $v \in \R^d$ we have that 
$\prob{X \sim D}{|\inner{v, X-\mus}| \geq t} \leq \exp(-t^2/2)$.
For this case, we show\footnote{To avoid clutter in the relevant expressions, 
all algorithms in this paper have high constant success probability. 
By standard techniques, the success probability can be boosted to $1-\tau$, for any $\tau>0$, 
at the cost of a $\log(1/\tau)$ increase in the sample complexity.}:

\begin{theorem} [Robust Mean Estimation for Sub-Gaussian Distributions] \label{thm:subg-mean}
Let $D$ be a sub-gaussian distribution on $\R^d$ with unknown mean $\mus$ and identity covariance.
Let $0 < \eps < 1/3$ and $\delta = O(\eps \sqrt{\log{1/\eps}})$.
Given an $\eps$-corrupted set of $N = \Omega(d/\delta^2)$ samples drawn from $D$, there is an algorithm that runs in time $\widetilde{O}(N d) /\poly(\eps)$ and outputs a hypothesis vector $\widehat{\mu}$ such that with probability at least $9/10$
it holds $\|\widehat{\mu} - \mus\|_2 \leq O(\delta) = O(\eps\sqrt{\log{1/\eps}})$.
\end{theorem}

It is well-known (see, e.g.,~\cite{DKK+17}) that the optimal error guarantee under the assumptions of
Theorem~\ref{thm:subg-mean} is $\Omega(\eps\sqrt{\log 1/\eps})$, even in the infinite sample regime.
Moreover, the sample complexity of the learning problem is known to be $\Omega(d/\delta^2)$ even without corruptions. 
Thus, our algorithm has best possible error guarantee and sample complexity, up to constant factors. 
Prior work~\cite{DKKLMS16, DKK+17} gave algorithms with the same error 
and sample complexity guarantees, but with runtime $\Omega(Nd^2)$, even for constant $\eps$. 
We note that for the very special case that $D = \mathcal{N}(\mus, I)$, an error of $O(\eps)$ is information-theoretically possible.  
However, as shown in~\cite{DKS17-sq}, any Statistical Query algorithm that runs in time $\poly(N)$ needs to have error
$\Omega(\eps \sqrt{\log(1/\eps)})$. Our algorithm achieves this accuracy guarantee in nearly-linear time.
See Section~\ref{ssec:prior} for a detailed summary of previous work.


%

Theorem~\ref{thm:subg-mean} handles the case that the covariance matrix of the good data distribution 
is known a priori. This is a somewhat limiting assumption. 
In our second main algorithmic result, we obtain a similarly robust algorithm 
under the much weaker assumption that the covariance matrix is unknown and bounded from above.
Specifically, we show:

\begin{theorem}[Robust Mean Estimation for Bounded Covariance Distributions] \label{thm:cov-mean}
Let $D$ be a distribution on $\R^d$ with unknown mean $\mus$ and unknown covariance matrix $\Sigma$ such that $\Sigma \preceq \sigma^2 I$.
Let $0 < \eps < 1/3$ and $\delta = O(\sqrt{\eps})$.
Given an $\eps$-corrupted set of $N = \Omega((d \log d) /\eps)$ samples drawn from $D$, there is an algorithm that runs in time $\widetilde{O}(N d) /\poly(\eps)$ and outputs a hypothesis vector $\widehat{\mu}$ such that with probability at least $9/10$ it holds $\|\widehat{\mu} - \mus\|_2 \leq O(\sigma \delta) = O(\sigma\sqrt{\eps})$.
\end{theorem}

Similarly, the sample complexity of our algorithm is best possible within a logarithmic factor, even without corruptions; the 
$O(\sigma \sqrt{\eps})$ error guarantee is known to be information-theoretically optimal, up to constants, 
even in the infinite sample regime. Previous algorithms~\cite{DKK+17,SteinhardtCV18} 
gave the same sample complexity and error guarantees, 
but again with significantly higher time complexities in high dimensions. Specifically, the iterative spectral 
algorithm of~\cite{DKK+17} has runtime $\Omega(N^2 \cdot d) = \Omega(d^3/\eps^2)$. 
See Section~\ref{ssec:prior} for more detailed comparisons.


We note that an efficient algorithm for robust mean estimation under bounded covariance assumptions 
has been recently used as a subroutine~\cite{PrasadSBR2018, DiakonikolasKKLSS2018sever} 
to obtain robust learners for a wide range of supervised learning problems 
that can be phrased as stochastic convex programs. This includes linear and logistic regression, 
generalized linear models, SVMs (learning linear separators under hinge loss), and many others.
The algorithm of Theorem~\ref{thm:cov-mean} provides a faster implementation of such a subroutine, 
hence yields faster robust algorithms for all these problems.

\subsection{Related and Prior Work} \label{ssec:prior}


Learning in the presence of outliers is an important goal in 
statistics and has been studied in the robust statistics community 
since the 1960s~\cite{Huber64}. After several decades of work, a number of sample-efficient and 
robust estimators have been discovered (see~\cite{Huber09, HampelEtalBook86} for book-length introductions).
For example, the Tukey median~\cite{Tukey75} is a sample-efficient robust mean estimator 
for various symmetric distributions \cite{Donoho92, CGR15b}. However, it is NP-hard 
to compute in general \cite{JP:78, AmaldiKann:95} and the many heuristics for computing it degrade 
in the quality of their approximation as the dimension scales \cite{Clarkson93, Chan04, MillerS10}. 

Until recently, all known computationally efficient high-dimensional estimators 
could only tolerate a negligible fraction of outliers, 
even for the simplest statistical task of mean estimation. Recent work in the theoretical
computer science community~\cite{DKKLMS16, LaiRV16} gave the first efficient robust
estimators for basic high-dimensional unsupervised tasks, including 
mean and covariance estimation. Since the dissemination of~\cite{DKKLMS16, LaiRV16}, 
there has been a flurry of research activity on robust learning algorithms in both supervised
and unsupervised settings~\cite{BDLS17, CSV17, DKK+17, DKS17-sq, DiakonikolasKKLMS18, SteinhardtCV18, DiakonikolasKS18-mixtures, DiakonikolasKS18-nasty, HopkinsL18, KothariSS18, PrasadSBR2018, DiakonikolasKKLSS2018sever, KlivansKM18, DKS19-lr, LSLC18-sparse, ChengDKS18}.

For the specific task of robust mean estimation,~\cite{DKKLMS16} designs two related algorithmic
techniques with similar sample complexities and error guarantees: a convex programming method
and an iterative spectral outlier removal method (filtering). The former method inherently relies on the ellipsoid
algorithm (leading to polynomial, yet impractical, runtimes), while the latter only requires repeated applications 
of power iteration to compute the highest eigenvalue-eigenvector of a covariance-like matrix. 
The total number of power iteration calls can be as large as $\Omega(d)$, for constant $\eps$, 
leading to runtimes of the form $\tilde{\Omega}(Nd^2)$. We note that the filter-based robust mean estimation 
algorithm, as presented in~\cite{DKKLMS16}, applies to the sub-gaussian case (as in Theorem~\ref{thm:subg-mean}). 
A slight variant of the method~\cite{DKK+17} applies under second 
moment assumptions (as in Theorem~\ref{thm:cov-mean}).  

The work~\cite{LaiRV16} gives a recursive dimension-halving technique with 
near-optimal accuracy, up to a logarithmic factor in the dimension. The aforementioned method 
requires computing the SVD of a second moment matrix $\Omega(\log d)$ times. Consequently, each iteration 
incurs runtime $\Omega(d^3)$. Similarly, the robust mean estimation algorithm under bounded second moments
in~\cite{SteinhardtCV18} requires computing the SVD of a matrix multiple times, leading to $\Omega(d^3)$ runtime.

\subsection{Our Approach and Techniques}
\label{subsec:approach}
In this section, we provide a detailed outline of our algorithmic approach
in tandem with a brief comparison to the most technically relevant prior work.
To robustly estimate the unknown mean $\mus$, we proceed as follows:
Starting with an initial guess $\nu$, in a sequence of iterations
we either certify that the current guess is close to the true mean $\mu^\star$
or refine our current guess with a new one that is provably closer to 
$\mu^\star$.

Let $D$ be the uncorrupted unknown distribution and $\Sigma$ be the covariance of the good samples. 
Then we know that the second order moment $\expect{X\sim D}{(X - \nu)(X-\nu)^\top}$ is equal to $\Sigma$ when $\nu = \mu^\star$, and is equal to $\Sigma + (\nu - \mu^\star)(\nu - \mu^\star)^\top$ in general. Therefore, the second order moment is minimized when $\nu = \mu^\star$. We use this property to distinguish whether our guess $\nu$ is close to $\mu^\star$. Of course, the input contains both good samples and bad (corrupted) samples, and the bad samples can change the first two moments significantly. To get around this problem, we try to reweight the samples: let $\Delta_{N,\eps}$ denote the following set
\[
\Delta_{N,\eps} = \left\{ w \in \R^N : \sum_{i=1}^N w_i = 1 \text{ and } 0 \le w_i \le \frac{1}{(1-\eps) N} \text{ for all } i \right\} \; .
\]
Our approach will try to minimize the weighted second order moment $\sum_{i=1}^N w_i (X_i-\nu)(X_i-\nu)^\top$ for all $w \in \Delta_{N,\eps}$, 
with the intended solution being assigning $1/|G|$ weight to all the good samples. This can be formalized as an SDP:

\begin{lp}
\mini{\lambda_{\max} \left(\sum_{i=1}^N w_i (X_i - \nu) (X_i - \nu)^\top\right)}
\st \con{w \in \Delta_{N, \eps}}
\end{lp}

This SDP is similar to the convex program used in \cite{DKKLMS16} but has some important conceptual 
differences that allow us to get a faster algorithm. The convex program in \cite{DKKLMS16} is essentially 
this SDP with $\nu = \mu^\star$. However, of course one cannot solve it directly 
as we do not know $\mu^\star$. To overcome this difficulty,~\cite{DKKLMS16} designs a separation oracle, 
which roughly corresponds to finding a direction of large variance. The whole convex programming algorithm 
in \cite{DKKLMS16} then relies on the ellipsoid algorithm and is therefore slow in high dimensions. 

In contrast, we fix a guess $\nu$ for the true mean in the SDP. Even though this $\nu$ may not be correct, 
we will establish a win-win phenomenon: either $\nu$ is a good guess of $\mu^\star$
in which case we get a good set of weights, or $\nu$ is far from $\mu^\star$ and we can efficiently find 
a new guess $\nu'$ that is closer to $\mu^\star$ by a constant factor.

More precisely, we will show that for any guess $\nu$ that is sufficiently close to the actual mean $\mu^\star$, 
the optimal value of the SDP is small. In this case, the weights $\{w_i\}_{i=1}^N$ computed by the SDP can be used 
to produce an accurate estimate of the mean: $\hat\mu_w = \sum_{i=1}^n w_i X_i$ (see Lemma~\ref{lem:wrong-mean-primal-nosol}).
Note that in this case the estimate $\hat\mu_w$ will be more accurate than the current guess $\nu$.
On the other hand, when the guess $\nu$ is far from $\mu^\star$, the optimal value of the SDP is large, 
and the optimal dual solution gives a certificate on why the second order moment $\sum_{i=1}^N w_i (X_i-\nu)(X_i-\nu)^\top$ 
cannot be small no matter how we re-weight the samples using $w\in \Delta_{N,\eps}$. 
Intuitively, the reason that the second moment matrix cannot have small spectral norm is because 
of the extra component $(\nu - \mu^\star)(\nu - \mu^\star)^\top$ in the expected second moment matrix. 
That is, the dual solution gives us information about $\nu-\mu^\star$ (Lemma~\ref{lem:good-dual-better-nu}).

So far, we have sketched our approach of reducing the algorithmic problem to solving a small number of SDPs.
To get a fast algorithm, we need to solve the primal and dual SDPs in nearly-linear time.
We achieve this by reducing them to covering/packing SDPs and using the solvers in \cite{AllenLO16,PengTZ16}. 
We note that these solvers rely on the matrix multiplicative weights update method (mirror descent), though we will not use this fact in our analysis.
The main technical challenge here is that the approximate solutions 
to the reduced SDPs may violate some of the original constraints 
(specifically, the resulting $w$ may not be in $\Delta_{N,\eps}$). We show that our main 
arguments are robust enough to handle these mild violations.

A perhaps surprising byproduct of our results is that a natural family of SDPs leads to asymptotically
faster algorithms for robust mean estimation than the previous fastest spectral algorithm~\cite{DKKLMS16} for
the most interesting parameter regime (corresponding to large dimension $d$ so that $d \gg \poly(1/\eps)$).
We view this as an interesting conceptual implication of our results: in our setting, principled SDP formulations 
can lead to faster runtimes compared to spectral algorithms, by exploiting the additional structure of these SDPs.
This phenomenon illustrates the value of obtaining a deeper understanding of such convex formulations.

\subsection{Structure of This Paper}
In Section~\ref{sec:mean}, we describe our algorithmic approach for robust mean estimation 
and use it to obtain our algorithm for sub-gaussian distributions (thus establishing Theorem~\ref{thm:subg-mean}).
In Section~\ref{sec:sdp}, we show that the corresponding SDPs can be solved in nearly-linear time.
In Section~\ref{sec:bounded}, we adapt our approach from Section~\ref{sec:mean} to 
obtain our algorithm for robust mean estimation under bounded covariance assumptions 
(thus establishing Theorem~\ref{thm:cov-mean}). For the clarity of the presentation, some proofs 
have been deferred to an appendix.

\section{Preliminaries}

For $n \in \mathbb{Z}_+$, we use $[n]$ to denote the set $\{1, \ldots, n\}$.
We use $e_i$ for the $i$-th standard basis vector, and $I$ for the identity matrix.
For a vector $x$, we use $\normone{x}$ and $\normtwo{x}$ to denote the $\ell_1$ and $\ell_2$ norms of $x$ respectively.
We use $\inner{x, y}$ to denote the inner product of two vectors $x$ and $y$: $\inner{x, y} = x^\top y = \sum_i x_i y_i$.

For a matrix $A$, we use $\normtwo{A}$ to denote the spectral norm of $A$, and $\lambda_{\max}$ to denote the maximum eigenvalue of $A$.
We use $\tr(A)$ to denote the trace of a square matrix $A$, and $\inner{A, B}$ or $A \bullet B$ for the entry-wise inner product of $A$ and $B$: $\inner{A, B} = A \bullet B = \tr(A^\top B)$.
A symmetric $n \times n$ matrix $A$ is said to be positive semidefinite (PSD) if for all vectors $x \in \R^n$, $x^\top A x \ge 0$.
For two symmetric matrices $A$ and $B$, we write $A \preceq B$ when $B - A$ is positive semidefinite.

Throughout this paper, we use $D$ to denote the ground-truth distribution.
We use $d$ for the dimension of $D$, $N$ for the number of samples, and $\eps$ for the fraction of corrupted samples.
Let $G^\star$ be the original set of $N$ uncorrupted samples drawn from $D$.
After the adversary corrupts an $\eps$-fraction of $G^\star$, we use $G \subseteq G^\star$ to denote the remaining set of good samples,
  and $B$ to denote the set of bad samples added by the adversary.
Note that $G \cup B$ is the input given to the algorithm, and we have $G \subseteq G^\star$, $|G| \ge (1-\eps)N$, and $|B| \le \eps N$.

We use $\mus$ to denote the (unknown) true mean of $D$, and $\nu$ to be our current guess for $\mus$.
We write $X_i$ for the $i$-th sample.
Both $\mus$, $\nu$, and the $X_i$'s are $d \times 1$ column vectors.
For a vector $w \in \R^N$, we define $w_G = \sum_{i \in G} w_i$ and $w_B = \sum_{i \in B} w_i$,
  and we use $\hat \mu_w = \sum_{i \in [N]} w_i X_i$ to denote the empirical mean weighted by $w$.

We call a vector $w \in \R^N$ a uniform distribution over a set $S \subseteq [N]$ if $w_i = \frac{1}{|S|}$ for all $i \in S$ and $w_i = 0$ otherwise.
Let $\Delta_{N,\eps}$ denote the convex hull of all uniform distributions over subsets $S \subseteq [N]$ of size $|S| = (1-\eps)N$. Formally,
$\Delta_{N,\eps} = \{w \in \R^N : \sum_i w_i = 1 \text{ and } 0 \le w_i \le \frac{1}{(1-\eps) N} \text{ for all } i\}.$


\section{Robust Mean Estimation for Known Covariance Sub-Gaussian Distributions}
\label{sec:mean}

In this section, we will describe our algorithmic technique and 
give an algorithm establishing Theorem~\ref{thm:subg-mean}. 



As we described in Section~\ref{subsec:approach}, our algorithm is going to make a guess $\nu$ 
for the actual mean $\mu^\star$, and try to certify its correctness by an SDP. 
In Section~\ref{subsec:alg}, we give the SDP formulation and describe the entire algorithm. 
In Section~\ref{subsec:sdpvalue}, we show that the optimal value of the primal/dual SDPs 
are closely related to the distance $\normtwo{\nu-\mu^\star}$. When the current guess $\nu$ is close to $\mu^\star$, 
we show (Section~\ref{subsec:primal}) that the solution to the primal SDP is going to give 
an accurate estimate of $\mu^\star$. On the other hand, 
when the current guess $\nu$ is far, in Section~\ref{subsec:dual} we 
analyze the dual solution and show how to find a new guess $\nu'$ that is closer to $\mu^\star$. 
Finally, we combine these techniques and prove Theorem~\ref{thm:subg-mean} 
in Section~\ref{subsec:mainproof}.

\subsection{SDP Formulation and Algorithm Description}
\label{subsec:alg}
As we mentioned in Section~\ref{subsec:approach}, we will use an SDP to try to certify that 
our current guess $\nu$ is close to the true mean $\mu^\star$. To achieve that, we assign weights 
$w_i$ to the samples while making sure that $w\in \Delta_{N,\epsilon}$. 
More precisely, the primal SDP with parameter $\nu\in \R^d$ and $\eps > 0$ is defined below:


\begin{lp}
\label{eqn:primal-sdp}
\mini{\lambda_{\max} \left(\sum_{i=1}^N w_i (X_i - \nu) (X_i - \nu)^\top\right)} 
\st \con{w \in \Delta_{N, \eps}}
\end{lp}

Intuitively, this SDP tries to re-weight the samples to minimize the second moment 
matrix $\sum_{i=1}^N w_i (X_i - \nu) (X_i - \nu)^\top$. The intended solution to this SDP is to assign 
weight $1/|G|$ on each of the good samples. This solution will have a small objective value 
whenever $\nu$ is close to $\mu^\star$.

When $\nu$ is far from $\mu^\star$, we need to consider the dual of \eqref{eqn:primal-sdp}.
We will first derive the dual of \eqref{eqn:primal-sdp}.
Note that the primal SDP is equivalent to the following:
\[ \min_{w \in \Delta_{N,\eps}} \max_{M \succeq 0, \tr(M)=1} \inner{M, \sum_i w_i (X_i - \nu)(X_i - \nu)^\top} \;. \]
Strong duality holds in our setting because the primal SDP admits a strictly feasible solution. The dual SDP can now
be written as follows:
\[ \max_{M \succeq 0, \tr(M)=1} \min_{w \in \Delta_{N,\eps}} \inner{M, \sum_i w_i (X_i - \nu)(X_i - \nu)^\top} \;. \]
Observe that once we fix a dual solution $M$, it is easy to minimize the objective function over $w$: 
{the minimum value is attained by assigning weight $w_i = \frac{1}{(1-\eps)N}$ to the smallest $(1-\eps)N$ inner products.
Therefore, the dual SDP can be equivalently written as follows:
\begin{lp}
\label{eqn:dual-sdp}
\maxi{\text{average of the smallest $(1-\eps)$-fraction of $\left((X_i - \nu)^\top M (X_i - \nu)\right)_{i=1}^N$}}
\st \con{M \succeq 0, \tr(M) \le 1}
\end{lp}


The dual SDP~\eqref{eqn:dual-sdp} certifies that there are no good weights that can make the spectral norm small. 
The intended solution for the dual is $M = yy^\top$, where $y = \frac{\nu-\mu^\star}{\normtwo{\nu-\mu^\star}}$ 
is the direction between $\nu$ and $\mu^\star$. Note that when $M = yy^\top$, 
the value $(X_i - \nu)^\top M (X_i - \nu)$ is exactly the squared norm of the projection in the direction $y$. 
Intuitively, if we project the samples onto the direction of $y$, the mean of the good samples is 
going to be at distance $\normtwo{\nu-\mu^\star}$, so 
even after removing the farthest $\eps$-fraction of the projected samples 
one cannot make the remaining values of $(X_i - \nu)^\top M (X_i - \nu)$ small. 
Of course, in general, the dual solution can be of rank higher than $1$, 
but we will show that any near-optimal dual solution must be close to rank $1$  
later in Section~\ref{subsec:dual}.

The SDPs are parameterized by $\eps > 0$ and $\nu \in \R^d$, 
which is our current guess of the true mean $\mus$.
We will solve both SDPs multiple times for different values of $\nu \in \R^d$, 
and we will update $\nu$ iteratively based on the solutions to previous SDPs.
Eventually, we will obtain some $\nu$ that is close enough to $\mus$, so that the primal SDP 
is going to provide a good set of weights $w$, and we can output the weighted 
empirical mean $\hat \mu_w = \sum_i w_i X_i$.



To avoid dealing with the randomness of the good samples, we require the following 
deterministic conditions on the original set of $N$ good samples $G^\star$ (which hold with probability $1-\tau$) drawn from the sub-gaussian distribution.
For all $w\in \Delta_{N,3\eps}$, we require the following conditions to hold for $\delta = c_1(\eps \sqrt{\log 1/\eps})$ and $\delta_2 = c_1(\eps \log 1/\eps)$ for some universal constant $c_1$:
\begin{align}
\label{eqn:good-sample-moments}
\normtwo{\sum_{i \in G^\star} w_i (X_i - \mus)} \le \delta \; , \; \quad
\normtwo{\sum_{i \in G^\star} w_i (X_i - \mus) (X_i - \mus)^\top - I} \le \delta_2 \; , \\
\forall i\in G^\star, \; \normtwo{X_i-\mus} \le O(\sqrt{d\log(N/\tau)}) \; . \label{eqn:good-pruning}
\end{align}
Intuitively, Equations \eqref{eqn:good-sample-moments} show that removing any $\eps$-fraction of good samples 
will not distort the mean and the covariance by too much. 
Equation \eqref{eqn:good-pruning} says that the good samples are not too far from the true mean.

We note that the above deterministic conditions are identical to the ones used in the 
convex programming technique of~\cite{DKKLMS16} to robustly learn the mean of $\mathcal{N}(\mus, I)$.
We note that the proof of these concentration inequalities does not require the Gaussian assumption, 
and it directly applies to sub-Gaussian distributions with identity covariance.
It follows from the analysis in~\cite{DKKLMS16} that 
after $N = \Omega(\delta^{-2} (d + \log(1/\tau)))$ samples, 
these conditions hold with probability at least $1-\tau$ on the set of good samples.

Throughout the rest of this section, we will assume that the above conditions are satisfied
where we set the parameter $\tau$ to be a sufficiently small universal constant; selecting $\tau = 1/30$ suffices for all our arguments.


We are now ready to present our algorithm (Algorithm~\ref{alg:pc-sdp}) to robustly estimate the mean of known covariance sub-gaussian distributions. 

\begin{algorithm}[h]
  \caption{Robust Mean Estimation for Known Covariance Sub-Gaussian}
  \label{alg:pc-sdp}
  \SetKwInOut{Input}{Input}
  \SetKwInOut{Output}{Output}
  \Input{An $\eps$-corrupted set of $N$ samples $\{X_i\}_{i=1}^N$ on $\R^d$ with $N = \tilde \Omega(d/\eps^2)$ and $\eps < 1/3$.}
  \Output{A vector $\hat \mu \in \R^d$ such that, with probability $9/10$, $\normtwo{\hat \mu - \mus} \le O(\eps \sqrt{\log(1/\eps)})$.}
  Let $\nu \in \R^d$ be the coordinate-wise median of $\{X_i\}_{i=1}^N$\;
  \For{$i = 1$ {\bf to} $O(\log d)$}{
   Use Proposition~\ref{prop:approx-sdp} to compute either \\
   {\em (i)} A good solution $w \in \R^N$ for the primal SDP~\eqref{eqn:primal-sdp} with parameters $\nu$ and $2\eps$; or \\
   {\em (ii)} A good solution $M \in \R^{d \times d}$ for the dual SDP~\eqref{eqn:dual-sdp} with parameters $\nu$ and $\eps$\;
   \eIf{the objective value of $w$ in SDP~\eqref{eqn:primal-sdp} is at most $1 + c_4 (\eps \ln(1/\eps))$}{
     \Return{the weighted empirical mean $\hat \mu_w = \sum_{i=1}^N w_i X_i$} (Lemma~\ref{lem:wrong-mean-primal-nosol})\;
   }{
     Move $\nu$ closer to $\mus$ using the top eigenvector of $M$ (Lemma~\ref{lem:good-dual-better-nu}).    }
  }
\end{algorithm}

\medskip 

\noindent {\bf Notation.} In this section, we will use $c_1, \ldots, c_7$ to denote universal constants that are independent of $N$, $d$, and $\eps$.
We will give a detailed description on how to set these constants in Appendix~\ref{apx:const}.


\subsection{Optimal Value of the SDPs}
\label{subsec:sdpvalue}

In this subsection, we will give upper and lower bounds on the optimal value of the SDPs \eqref{eqn:primal-sdp} and \eqref{eqn:dual-sdp}.
Recall that our high-level idea is to use the dual SDP to improve our guess $\nu$, until it is close enough to the true mean $\mus$, and then solve the primal SDP to get a good set of weights.
However, we cannot write an if statement based on $r = \normtwo{\nu - \mus}$ because we do not know $\mus$.

Lemma~\ref{lem:ub-lb-opt} allows us to estimate $r$ from the optimal value of the SDPs.
We will bound the optimal value of the SDPs from both sides using feasible primal and dual solutions.
Let $\OPT_{\nu,\eps}$ denote the optimal value of the SDPs~\eqref{eqn:primal-sdp},~\eqref{eqn:dual-sdp} 
with parameters $\nu$ and $\eps$. The following lemma shows that when $\eps$ is small 
and $\nu$ is far away from $\mus$, then both the optimal values $\OPT_{\nu,\eps}$ and $\OPT_{\nu,2\eps}$ 
are close to $1 + \normtwo{\mus-\nu}^2$. 

\begin{lemma}[Optimal Value of the SDPs]
\label{lem:ub-lb-opt}
Fix $0 < \eps < 1/3$ and $\nu \in \R^d$.
Let $\delta = c_1 \eps \sqrt{\ln(1/\eps)}$, $\delta_2 = c_1 \eps \ln(1/\eps)$ and $\beta = \sqrt{\eps \ln(1/\eps)}$.
Let $\{X_i\}_{i=1}^N$ be an $\eps$-corrupted set of $N = \Omega(d/\delta^2)$ samples drawn from a sub-gaussian distribution
with identity covariance. 
Let $\OPT_{\nu,\eps}$ denote the optimal value of the SDPs~\eqref{eqn:primal-sdp},~\eqref{eqn:dual-sdp} 
with parameters $\nu$ and $\eps$.
Let $r = \normtwo{\nu - \mus}$.
Then, we have:
\begin{align*}
(1 - \delta_2) + r^2 - 2 \delta r & \le \OPT_{\nu,2\eps} 
  \le \OPT_{\nu,\eps} \le (1 + \delta_2) + r^2 + 2\delta r \; .
\end{align*}
In particular, when $r \ge c_2 \beta$, we can simplify the above as
\[
1 + 0.9 r^2 \le \OPT_{\nu,2\eps} \le \OPT_{\nu,\eps} \le 1 + 1.1 r^2 \; .
\]
\end{lemma}
\begin{proof}
We first prove the argument for $\OPT = \OPT_{\nu, \eps}$.

One feasible primal solution is to set $w_i = \frac{1}{|G|}$ for all $i \in G$ (and $w_i = 0$ for all $i \in B$).
Therefore, 
\begin{align*}
\OPT & \le \lambda_{\max} \left( \sum_{i=1}^N w_i (X_i - \nu) (X_i - \nu)^\top \right) 
  = \max_{y \in \R^d, \normtwo{y}=1} \sum_{i\in G} w_i \inner{X_i - \nu, y}^2 \\
  & = \max_{y \in \R^d, \normtwo{y}=1} \left(\sum_{i\in G} w_i \inner{X_i - \mus, y}^2 + \inner{\mus - \nu, y}^2 + 2\inner{\sum_{i\in G} w_i (X_i - \mus), y}\inner{\mus - \nu, y}\right) \\
  & \le \max_{y \in \R^d, \normtwo{y}=1} \left((1+\delta_2) + \inner{\mus - \nu, y}^2 + 2\delta \inner{\mus - \nu, y}\right) \\
  & = (1 + \delta_2) + \normtwo{\mus - \nu}^2 + 2\delta \normtwo{\mus - \nu} \; .
\end{align*}
Notice that $w$ can be viewed as a weight vector on $G^\star$ and we have $w \in \Delta_{N,\eps}$.
This allows us to use Condition~\eqref{eqn:good-sample-moments} in the second to last step.

One feasible dual solution is $M = yy^\top$ where $y = \frac{\mus - \nu}{\normtwo{\mus - \nu}}$.
The dual objective value is the mean of the smallest $(1-\eps)$-fraction of $\left((X_i-\nu)^\top M (X_i-\nu)\right)_{i=1}^N$, which is at least
\[
\frac{1}{(1-\eps)N} \min_{S \subset G, |S| = (1-2\eps)N} \sum_{i\in S} (X_i-\nu)^\top M (X_i-\nu) \; .
\]
This is because $|G| = (1-\eps)N$, the smallest $(1-\eps)N$ entries must include $S$, where $S$ is the smallest $(1-2\eps)N$ entries in $G$.
Let $w'_i = \frac{1}{|S|}$ for all $i \in S$ and $w'_i = 0$ otherwise.
Note that $S \subset G$ and $|S| = (1-2\eps)N$, so $w'$ can be viewed as a weight vector on $G^\star$ with $w' \in \Delta_{N,2\eps}$.
Therefore we have 
\begin{align*}
\OPT & \ge \sum_{i\in S} \frac{1}{|S|} (X_i-\nu)^\top M (X_i-\nu) = \sum_{i \in G} w'_i \inner{X_i - \nu, y}^2 \\
  & = \sum_{i \in G} w'_i \inner{X_i - \mus, y}^2 + w'_G \normtwo{\mus - \nu}^2 + 2 \sum_{i \in G} w'_i \inner{X_i - \mus, y} \normtwo{\mus - \nu} \\
  & \ge (1 - \delta_2) + \normtwo{\mus - \nu}^2 - 2 \delta \normtwo{\mus - \nu} \; . 
\end{align*}

Now we consider $\OPT_{\nu, 2\eps}$.
Intuitively, $\OPT_{\nu, 2\eps} \approx \OPT_{\nu, \eps}$ because both SDPs can throw away the bad samples first, and whether we allow them to throw away another $\eps$-fraction of good samples should not affect the moments too much.
It is easy to see that $\OPT_{\nu, 2\eps} \le \OPT_{\nu, \eps}$, because the feasible region with parameter $2\eps$ is strictly larger ($\Delta_{N, 2\eps} \supset \Delta_{N,\eps}$) for the primal SDP.

It remains to show that the same lower bound holds for $\OPT_{\nu, 2\eps}$.
For the dual SDP with parameter $2\eps$, the objective is the mean of the smallest $(1-2\eps)$-fraction of the entries, so we pick $S$ to be the smallest $(1-3\eps)N$ entries in $G$ and $w'_i = \frac{1}{(1-3\eps)N}$ for all $i\in S$ instead.
Note that Condition~\eqref{eqn:good-sample-moments} holds for all $w \in \Delta_{N,3\eps}$, and the rest of the proof is identical.

To obtain the simpler upper and lower bounds when $r \ge c_2 \beta$, we note that the error term $\delta_2 + 2\delta r = \Theta(\eps \log(1/\eps)) = \Theta(r^2)$, so by increasing $c_2$ we can get $1 + 0.9 r^2 \le \OPT \le 1 + 1.1 r^2$.
\end{proof}

\subsection{When Primal SDP Has Good Solutions}
\label{subsec:primal}
In this section, we show that a good primal solution for {\em any} guess $\nu$ will give an accurate weighted empirical mean.
Lemma~\ref{lem:wrong-mean-primal-nosol} proves the contrapositive statement: 
if the weighted empirical mean $\hat \mu_w$, with respect to weight-vector $w$,
is far from the true mean, then no matter what our current guess $\nu$ is, $w$ cannot be a good solution to the primal SDP.
More specifically, we show that the objective value of $w$ is at least $1 + \Omega(\delta^2/\eps)$.
Roughly speaking, we get a contribution of $1$ from the good samples and a contribution of 
$\Omega(\delta^2/\eps)$ from the bad samples.

We briefly explain why the bad samples contribute $\Omega(\delta^2/\eps)$.
The empirical mean of the good samples is off by at most $\delta$ by Condition~\eqref{eqn:good-sample-moments}.
Now if $\hat \mu_w$ is far away from $\mus$, the bad samples must shift the mean by more than $\Omega(\delta)$.
Intuitively, if an $\eps$-fraction of the samples distort the mean by $\delta$, on average each of these sample contributes an error of $\delta/\eps$, which introduces a total error of $\eps(\delta/\eps)^2 = \delta^2/\eps$ in the second moment matrix.

We use $\beta = \sqrt{\eps \ln(1/\eps)} = \Theta(\sqrt{\delta^2/\eps})$ to denote (asymptotically) 
the distance between $\nu$ and $\mus$ at the end of our algorithm.
This threshold appears naturally because if $\normtwo{\nu - \mus} \gg \beta$, 
then Lemma~\ref{lem:ub-lb-opt} tells us that $\OPT - 1 \gg \beta^2 = \delta^2/\eps$.
This error subsumes the potential error we could get due to the bad samples 
shifting the mean by more than $\Omega(\delta)$, so we must guess some $\nu$ that is $O(\beta)$ 
from $\mus$ to detect the bad samples.
Note that given some $\nu$ at distance $O(\beta)$ to $\mus$, an optimal 
solution $w$ to the primal SDP will give a much better estimate $\hat \mu_w$ 
that is $O(\delta) \ll O(\beta)$ close to $\mus$.

\begin{lemma}[Good Primal Solution $\Rightarrow$ Correct Mean]
\label{lem:wrong-mean-primal-nosol}
Fix $0 < \eps < 1/3$.
Let $\delta = c_1 \eps \sqrt{\ln(1/\eps)}$, $\delta_2 = c_1 \eps \ln(1/\eps)$ and $\beta = \sqrt{\eps \ln(1/\eps)}$.
Let $\{X_i\}_{i=1}^N$ be a set of $\eps$-corrupted samples drawn from a sub-gaussian distribution
with identity covariance, where $N = \Omega(d/\delta^2)$.
For all $w \in \Delta_{N,2\eps}$, if $\normtwo{\hat \mu_w - \mus} \ge c_3 \delta$ where $\hat \mu_w = \sum_{i=1}^N w_i X_i$, then for all $\nu \in \R^d$,
\[
\lambda_{\max} \left( \sum_{i=1}^N w_i (X_i - \nu) (X_i - \nu)^\top \right) \ge 1 + c_4 \beta^2 \; . 
\]
\end{lemma}
\begin{proof}
Fix any $w \in \Delta_{N, 2\eps}$.
If $\normtwo{\mus - \nu} \ge c_5 \beta$, then because $w$ is feasible and by Lemma~\ref{lem:ub-lb-opt},
\[
\lambda_{\max} \left( \sum_{i=1}^N w_i (X_i - \nu) (X_i - \nu)^\top \right) \ge \OPT_{\nu,2\eps} \ge 1 + 0.9 \normtwo{\mus - \nu}^2 \ge 1 + 0.9 c_5^2 \beta^2 \ge 1 + c_4 \beta^2 \; .
\]
Therefore, for the rest of this proof, we can assume $\normtwo{\mus - \nu} < c_5 \beta$.

We project the samples along the direction of $(\hat \mu_w - \mus)$.
Consider the unit vector $y = (\hat \mu_w - \mus) / \normtwo{\hat \mu_w - \mus}$.
To bound from below the maximum eigenvalue, it is sufficient to show that
\[
y^\top \left( \sum_{i=1}^N w_i (X_i - \nu) (X_i - \nu)^\top \right) y = \sum_{i=1}^N w_i \inner{X_i - \nu, y}^2 \ge 1 + \Omega(\delta^2/\eps) \; .
\]
We first bound from below the contribution of the bad samples by $\Omega(\delta^2/\eps)$.
By triangle inequality, 
\begin{align*}
\abs{\sum_{i\in B} w_i \inner{X_i - \nu, y}}
  & \ge \abs{\sum_{i\in B} w_i \inner{X_i - \mus, y}} - w_B \abs{\inner{\mus - \nu, y}} \\
  & \ge \abs{\sum_{i=1}^N w_i \inner{X_i - \mus, y}} - \abs{\sum_{i\in G} w_i \inner{X_i - \mus, y}} - 2\eps \normtwo{\mus - \nu} \\
  & \ge \normtwo{\hat \mu_w - \mus} - \delta - 2 \eps c_5 \beta \ge (c_3 - 1 - 2 c_5 \frac{\sqrt{\eps}}{c_1}) \delta \ge c_6 \delta \; .
\end{align*}
The last line follows from our choice of $y$, and the good samples satisfy Condition~\eqref{eqn:good-sample-moments}.
By Cauchy-Schwarz,
$
\left(\sum_{i\in B} w_i \inner{X_i - \nu, y}^2\right) \left(\sum_{i\in B} w_i\right) \ge \left(\sum_{i\in B} w_i \inner{X_i - \nu, y} \right)^2 \ge c_6^2 \delta^2.
$
Since $w_B \le 2\eps$, we have $\sum_{i\in B} w_i \inner{X_i - \nu, y}^2 \ge \frac{c_6^2}{2} (\delta^2/\eps)$.

We continue to lower bound the contribution of the good samples to the quadratic form by $1 - O(\delta_2) = 1 - O(\delta^2/\eps)$.
This is because the true covariance matrix is $I$.
By Condition~\eqref{eqn:good-sample-moments},
\begin{align*}
\sum_{i\in G} w_i \inner{X_i - \nu, y}^2
  & = \sum_{i\in G} w_i \left(\inner{X_i - \mus, y}^2 + \inner{\mus - \nu, y}^2 + 2\inner{X_i - \mus, y}\inner{\mus - \nu, y}\right) \\
  & \ge \sum_{i\in G} w_i \inner{X_i - \mus, y}^2 + 2 \inner{\mus - \nu, y} \inner{\sum_{i\in G} w_i (X_i - \mus), y} \\
  & \ge (1 - \delta_2) - 2\delta \normtwo{\mus - \nu} \ge 1 - (1 + \frac{2c_5\beta}{\sqrt{\ln(1/\eps)}}) \delta_2 \ge 1 - c_7 \delta_2 \; .
\end{align*}
Putting the good and bad samples together, we have $\sum_{i=1}^N w_i \inner{X_i - \nu, y}^2 \ge 1 - c_7 \delta_2 + \frac{c_6^2}{2} (\delta^2/\eps) = 1 + (\frac{c_1^2 c_6^2}{2} - c_1 c_7) \beta^2 \ge 1 + c_4 \beta^2$ as needed.

The constants in the proof are given in Appendix~\ref{apx:const}.
\end{proof}

Lemma~\ref{lem:wrong-mean-primal-nosol} guarantees that any good solution to the primal SDP gives a good set of weights.
In other words, whenever we have a solution to the primal SDP whose objective value is at most $1 + O(\beta^2)$, we are done because the weighted empirical mean must be close to the true mean.

\subsection{When Primal SDP Has No Good Solutions}
\label{subsec:dual}
We now deal with the other possibility:  the primal SDP has no good solution.
We will show that, in this case, we can move $\nu$ closer to $\mus$ 
by solving the dual SDP~\eqref{eqn:dual-sdp}, decreasing $\normtwo{\nu - \mus}$ by a constant factor.

Lemma~\ref{lem:ub-lb-opt} states that $\OPT_{\nu, \eps} \approx 1 + \normtwo{\nu-\mus}^2$.
Intuitively, if the dual SDP throws away all the bad samples, then we know that 
$\OPT_{\nu, \eps} \approx\frac{1}{|G|} \sum_{i\in G} (X_i - \nu)^\top M (X_i - \nu)$.
If this quantity also concentrates around its expectation, then
\[
1 + \normtwo{\nu-\mus}^2 \approx \OPT_{\nu, \eps} \approx \expect{X \sim \NN(\mus,I)}{(X - \nu)^\top M (X - \nu)} = \inner{M, I + (\nu-\mus)(\nu-\mus)^\top} \; .
\]
Because $\tr(M) = 1$, we can remove $1$ from both sides and get $\inner{M, (\nu - \mus)(\nu - \mus)^\top} \approx \normtwo{\nu - \mus}^2$.
This condition implies that the top eigenvector of $M$ aligns approximately with $(\nu - \mus)$, which provides a good direction for us to move $\nu$.

The following lemma formalizes this intuition.
Specifically, Lemma~\ref{lem:good-dual-better-nu} shows that despite the error from solving the SDP 
approximately and the errors in the concentration inequalities, 
we can still use the top eigenvector of $M$ to move $\nu$ closer to $\mus$.

\begin{lemma}[Good Dual Solution $\Rightarrow$ Better $\nu$]
\label{lem:good-dual-better-nu}
Fix $0 < \eps < 1/3$ and $\nu \in \R^d$.
Let $\beta = \sqrt{\eps \ln(1/\eps)}$.
Assume we have a solution $M \in \R^{d \times d}$ to the dual SDP~\eqref{eqn:dual-sdp} with parameters $\nu$ and $\eps$, and the objective value of $M$ is at least $\max(1 + 0.9 c_4 \beta^2, (1-\frac{\eps}{10}) \OPT_{\nu, 2\eps})$.
Then, we can efficiently find a vector $\nu' \in \R^d$,
such that $\normtwo{\nu' - \mus} \le \frac{3}{4} \normtwo{\nu - \mus}$.
\end{lemma}
\begin{proof}
Because $M$ is a feasible solution to the dual SDP~\eqref{eqn:dual-sdp} with parameters $\nu$ and $\eps$, we know that $\OPT_{\nu,\eps} \ge 1 + 0.9 c_4 \beta^2$.
When $\OPT_{\nu, \eps} \ge 1 + 0.9 c_4 \beta^2$, Lemma~\ref{lem:ub-lb-opt} implies that $\normtwo{\mus - \nu} \ge c_2 \beta$ and $(1-\frac{\eps}{10})\OPT_{\nu, 2\eps} \ge 1 + 0.85 \normtwo{\mus - \nu}^2$.
Since the objective value is the average of the smallest $(1-\eps)N$ entries of $(X_i - \nu)^\top M (X_i - \nu)$, and one way to choose $(1-\eps)N$ entries is to focus on the good samples,
\[
1 + 0.85 \normtwo{\mus - \nu}^2
   \le \left(1-\frac{\eps}{10}\right) \OPT_{\nu, 2\eps}
   \le \left(1-\frac{\eps}{10}\right) \OPT_{\nu, \eps}
   \le \frac{1}{|G|}\sum_{i\in G} (X_i - \nu)^\top M (X_i - \nu) \; .
\]

We know $M \succeq 0$ and $\tr(M) = 1$.
Without loss of generality, we can assume $M$ is symmetric.
Using Condition~\eqref{eqn:good-sample-moments}, we can prove that $\inner{M, (\mus - \nu)(\mus - \nu)^\top} \ge \frac{3}{4} \normtwo{\mus-\nu}^2$:
\begin{align*}
1 + 0.85 \normtwo{\mus - \nu}^2
 & \le \frac{1}{|G|}\sum_{i\in G} (X_i - \nu)^\top M (X_i - \nu) \\
 & = \frac{1}{|G|}\sum_{i\in G} \inner{M, (X_i - \mus)(X_i - \mus)^\top + 2(X_i - \mus)(\mus-\nu) + (\mus-\nu)(\mus-\nu)^\top} \\
 & \le 1 + \delta_2 + 2\delta \normtwo{\mus-\nu} + \inner{M, (\mus-\nu)(\mus-\nu)^\top} \\
 & \le 1 + 0.1 \normtwo{\mus-\nu}^2 + \inner{M, (\mus-\nu)(\mus-\nu)^\top} \; .
\end{align*}

We will continue to show that the top eigenvector of $M$ aligns with $(\nu - \mus)$.
Let $\lambda_1 \ge \lambda_2 \ge \ldots \ge \lambda_d \ge 0$ denote the eigenvalues of $M$, 
and let $v_1, \ldots, v_d$ denote the corresponding eigenvectors.
The conditions on $M$ implies that $\sum_{i=1}^d \lambda_d = 1$.
We decompose $(\mus - \nu)$ and write it as $\mus - \nu = \sum_{i=1}^d \alpha_i v_i$ where $\sum_{i=1}^d \alpha_i^2 = \normtwo{\mus - \nu}^2$.
Using these decompositions, we can rewrite $\inner{M, (\mus - \nu)(\mus - \nu)^\top} = \sum_{i=1}^d \lambda_i \alpha_i^2$.

First observe that $\lambda_1 \ge \frac{3}{4}$, because $\lambda_1 \sum_{i} \alpha_i^2 \ge \sum_{i} \lambda_i \alpha_i^2 \ge \frac{3}{4} \normtwo{\mus - \nu}^2 = \frac{3}{4}\sum_{i} \alpha_i^2$.
Moreover, because $\frac{3}{4} \sum_{i} \alpha_i^2 \le \sum_{i} \lambda_i \alpha_i^2 \le \lambda_1 \alpha_1^2 + (1-\lambda_1)(1-\alpha_1^2) \le \frac{3}{4} \alpha_1^2 + \frac{1}{4} \sum_i \alpha_i^2$, we know that $\inner{v_1 v_1^\top, (\mus - \nu)(\mus - \nu)^\top} = \alpha_1^2 \ge \frac{2}{3} \sum_{i} \alpha_i^2$.
Thus, we have a unit vector $v_1 \in \R^d$ with $\inner{v_1, \mus - \nu} = \alpha_1 \ge \sqrt{2/3} \normtwo{\mus - \nu}$, so the angle between $v_1$ and $\mus - \nu$ is at most $\theta \le \cos^{-1}(\sqrt{2/3})$.

\begin{figure}[h]
\centering
\includegraphics[width=0.6\linewidth]{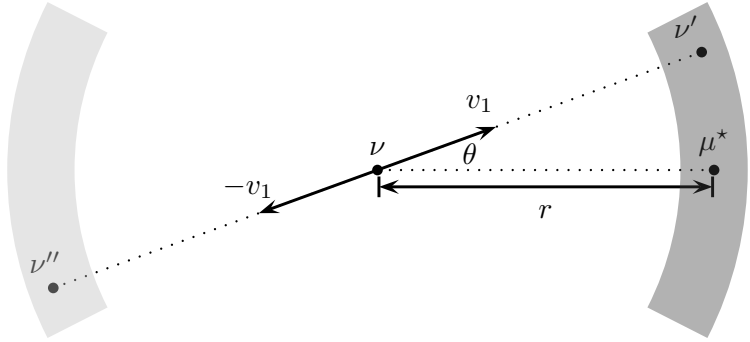}
\caption{An illustration of the final part of the proof of Lemma~\ref{lem:good-dual-better-nu}. Assume we can find a unit vector $v_1$ that approximately aligns with $(\mus - \nu)$, and we can estimate $r \approx \normtwo{\mus-\nu}$. Then, the point $\nu' = \nu + r' v_1$ lies in the highlighted region on the right, which is closer to $\mus$. Moreover, if only know $\pm v_1$, we can distinguish $\nu'$ and $\nu''$ by comparing the optimal value of their SDPs.} 
\label{fig:move-nu}
\end{figure}

Finally, if we know the exact value of $r = \normtwo{\mus-\nu}$, we can update $\nu$ to $\nu' = \nu + r v_1$.
This corresponds to moving $\nu$ to a point that is on a circle of radius $r$ centered at $\nu$ (see Figure~\ref{fig:move-nu}).
The distance between $\nu'$ and $\mus$ is maximized when $\theta$ is the largest, and this distance is at most $2r\sin(\theta/2) < \frac{2}{3} r$.
However, in reality, we do not know $r = \normtwo{\mus-\nu}$, and we can only estimate it from the value of $\OPT_{\nu,2\eps}$.
Because we are solving the SDPs to precision $(1 \pm \frac{\eps}{10})$, by Lemma~\ref{lem:ub-lb-opt}, 
we can estimate $r'$ such that $0.85 r^2 \le (r')^2 \le 1.15 r^2$.
By triangle inequality, the point $\nu' = \nu + r' v_1$ is at most $\frac{2}{3} r + |r'-r| < \frac{3}{4} r$ away from $\mus$.

One technical issue is that the top eigenvector of $M$ can be $\pm v_1$, so we have two possible directions that are opposite of each other.
Let $\nu' = \nu + r' v_1$ be the point closer to $\mus$, and $\nu'' = \nu - r' v_1$ be the point farther from $\mus$.
We can distinguish $\nu'$ and $\nu''$ by solving the SDP~\eqref{eqn:primal-sdp} with parameters $\nu'$ and $\nu''$ respectively, and the point with smaller optimal value is $\nu'$.
This is because $\nu''$ moves at least $r' \ge \sqrt{0.9} r$ in the reverse direction, so the distance between $\nu''$ and $\mus$ is at least $\sqrt{(r + r' \cos \theta)^2 + (r' \sin \theta)^2} > 1.8 r > c_2 \beta$.
By Lemma~\ref{lem:ub-lb-opt}, $\OPT_{\nu'',2\eps} \ge 1 + 0.9 \cdot (1.8r)^2 > 1 + 2r^2$, and $\OPT_{\nu',2\eps} \le 1 + 1.1 r^2$.
Again because $r \ge c_2 \beta$, this gap is large enough for separating them if we approximate both $\OPT_{\nu', 2\eps}$ and $\OPT_{\nu'', 2\eps}$ to a factor of $(1 \pm \frac{\eps}{10})$.

The constants in the proof are given in Appendix~\ref{apx:const}.
\end{proof}

\subsection{Proof of Theorem~\ref{thm:subg-mean}}
\label{subsec:mainproof}
We are now ready to prove Theorem~\ref{thm:subg-mean}. 
This is mostly done by applying Lemmas~\ref{lem:wrong-mean-primal-nosol} 
and~\ref{lem:good-dual-better-nu} in appropriate scenarios. 
Because of the geometric improvement in Lemma~\ref{lem:good-dual-better-nu}, 
we will apply it at most a logarithmic number of times, 
and then the algorithm can terminate in the case of Lemma~\ref{lem:wrong-mean-primal-nosol}.

\begin{proof}[{\bf Proof of Theorem~\ref{thm:subg-mean}} (Correctness and Runtime of Algorithm~\ref{alg:pc-sdp})]
Let $\tau = 1/30$.
When $N = \Omega(d/\delta^2)$, Condition~\eqref{eqn:good-sample-moments} holds for the good samples with probability at least $1-\tau$, 
which is required in the proofs of Lemmas~\ref{lem:ub-lb-opt},~\ref{lem:wrong-mean-primal-nosol},~and~\ref{lem:good-dual-better-nu}.

We will use the empirical coordinate-wise median as our initial guess $\nu$.
It is folklore that with high probability, the coordinate-wise median is within $O(\eps \sqrt{d})$ of the true mean $\mus$.
In Algorithm~\ref{alg:pc-sdp}, whenever we update $\nu$ by Lemma~\ref{lem:good-dual-better-nu}, we move it closer to $\mus$.
Therefore, throughout the algorithm,  the condition $\normtwo{\nu - \mus} \le O(\eps \sqrt{d})$ always holds, which is required by Proposition~\ref{prop:approx-sdp}.

The correctness of Algorithm~(\ref{alg:pc-sdp}) follows immediately from 
Lemmas~\ref{lem:wrong-mean-primal-nosol},~\ref{lem:good-dual-better-nu}, and Proposition~\ref{prop:approx-sdp}.
In each iteration, the algorithm either finds a good solution $w \in \R^N$ to the primal SDP~\eqref{eqn:primal-sdp} and terminates, 
in which case Lemma~\ref{lem:wrong-mean-primal-nosol} guarantees that the weighted empirical mean $\hat \mu_w$ is close to $\mus$;
or the algorithm finds a good solution $M \in \R^{d\times d}$ to the dual SDP~\eqref{eqn:dual-sdp}, 
and it will use the top eigenvector of $M$ to move the current guess $\nu$ closer to $\mus$ 
by a constant factor, as in Lemma~\ref{lem:good-dual-better-nu}.
The failing probability is at most $3\tau = 1/10$ by a union bound over three bad events: {\em (i)} the good samples do not satisfy Condition~\eqref{eqn:good-sample-moments}, {\em (ii)} the coordinate-wise median is too far away from $\mus$, and {\em (iii)} the SDP solver is not able to produce an approximate solution at some point.

We now analyze the running time of Algorithm~\ref{alg:pc-sdp}.
The coordinate-wise median can be computed in time $O(Nd)$.
Whenever the primal SDP~\eqref{eqn:primal-sdp} has a good solution, the algorithm terminates.
The initial choice of $\nu$ satisfies that $\normtwo{\nu - \mus} \le O(\eps \sqrt{d})$, 
and Lemma~\ref{lem:ub-lb-opt} implies that we have a good primal solution if $\normtwo{\mus - \nu} \le O(\beta)$.
Because every time we move $\nu$ as in Lemma~\ref{lem:good-dual-better-nu}, the distance between $\nu$ and $\mus$ 
decreases by a constant factor, we can move $\nu$ 
at most $O(\log(\eps \sqrt{d}/\beta)) = O(\log d)$ times.
For each guess $\nu \in \R^d$, we invoke Proposition~\ref{prop:approx-sdp} to either obtain a good primal or a good dual solution.
We repeat every use of Proposition~\ref{prop:approx-sdp} $O(\log \log d)$ times, so that the failure 
probability is at most $\tau = 1/30$ by a union bound over the iterations.
We will use the power method to compute the top eigenvector of $M$, which takes time $O(\log d \cdot N d \log^2 N/\eps^3)$~(see Remark~\ref{rem:pm}).
Thus, every loop of Algorithm~\ref{alg:pc-sdp} takes time $O(\log \log d) \cdot (\tilde O(N d / \eps^6) + \tilde O(N d /\eps^3)) = \tilde O(N d / \eps^6)$.
Therefore, the overall running time is 
\[ O(Nd) + O(\log d) \cdot \tilde O(N d / \eps^6) = \tilde O(N d  / \eps^6) \;. \qedhere \]
\end{proof}

\begin{remark} \label{rem:pm}
{\em Recall that the number of iterations of the power method is $O(\log d / \eps')$ if we want to compute a $(1-\eps')$-approximate largest eigenvector.
Due to the slack in the geometry analysis of Lemma~\ref{lem:good-dual-better-nu}, we can set $\eps'$ to a constant (say $\eps' = 0.01$).
In addition, the matrix $M$ is given implicitly by the positive SDP solver (e.g.,~\cite{PengTZ16}) 
as the sum of matrix exponentials $M = \frac{1}{T}\sum_{t=1}^T \frac{W_t}{\tr(W_t)}$, where $T = O(\log^2 N / \eps^3)$ is the number of iterations 
of the positive SDP solver and $W_t = \exp(\sum_{i=1}^N x^t_i A_i A_i^\top)$ for some $x^t \in \R^N$.
To evaluate $M v$ in the power method, we multiply $v$ with each $W_t$ separately, where we use a degree $
O(\log(1/\eps'))$ matrix polynomial of $\Psi_t = \sum_{i=1}^N x^t_i A_i A_i^\top$ to approximate $W_t = \exp(\Psi_t)$.
It takes time $O(N d)$ to compute $\Psi_t v$, and therefore it takes time $O(T N d) = O(N d \log^2 N / \eps^3)$ to evaluate $M v$.}
\end{remark}


\section{Solving Primal/Dual SDPs in Nearly-Linear Time}
\label{sec:sdp}
By combining Lemmas~\ref{lem:wrong-mean-primal-nosol}~and~\ref{lem:good-dual-better-nu} from Section~\ref{sec:mean}, 
we know that we can make progress by either finding any solution to the primal SDP~\eqref{eqn:primal-sdp} 
with objective value at most $1 + c_4 \beta^2$, or by finding an approximately optimal solution to the dual 
SDP~\eqref{eqn:dual-sdp} whose objective value is at least $1 + \frac{9}{10} c_4 \beta^2$.
This section is dedicated to proving Proposition~\ref{prop:approx-sdp}, 
which shows that this can be done in time $\tilde O(N d) / \poly(\eps)$.

\begin{proposition}
\label{prop:approx-sdp}
Fix $0 < \eps < 1/3$, and $\nu \in \R^d$ with $\normtwo{\nu - \mus} \le O(\eps \sqrt{d})$.
Let $\beta = \sqrt{\eps \ln(1/\eps)}$.
We can compute in time $\tilde O(N d/ \eps^6)$, with probability at least $9/10$, either
\begin{enumerate}
\item A solution $w \in \R^{N}$ to the primal SDP~\eqref{eqn:primal-sdp} with parameters $(\nu, 2\eps)$, such that the objective value of $w$ is at most $1 + c_4\beta^2$; or
\item A solution $M \in \R^{d \times d}$ to the dual SDP~\eqref{eqn:dual-sdp} with parameter $(\nu, \eps)$, such that the objective value of $M$ is at least $\max(1 + \frac{9}{10}c_4 \beta^2, (1-\frac{\eps}{10}) \OPT_{\nu,2\eps})$.
\end{enumerate}
\end{proposition}

Previously, nearly-linear time SDP solvers were developed for packing/covering SDPs \cite{AllenLO16,PengTZ16}. 
At a high level, we first relate SDPs~\eqref{eqn:primal-sdp},~\eqref{eqn:dual-sdp} 
with a pair of packing/covering SDPs~\eqref{eqn:primal-standard},~\eqref{eqn:dual-standard}, 
where we switch the objective function with some constraint and introduce an additional parameter $\rho > 0$.
Next, we show that to prove Proposition~\ref{prop:approx-sdp}, 
it is sufficient to solve SDPs~\eqref{eqn:primal-standard},~\eqref{eqn:dual-standard} 
approximately for the correct value of $\rho$,
and moreover, we can run binary search to find a suitable $\rho$. 
Finally, in Section~\ref{sec:sdp-solver}, we show that our packing/covering 
SDPs~\eqref{eqn:primal-standard},~\eqref{eqn:dual-standard} can be solved 
in time $\tilde O(N d / \eps^6)$. Note that these running times (specifically, the dependence on $\eps$) 
can be improved if better packing/covering SDP solvers are discovered.
For example, \cite{AllenLO16} mentioned the possibility of achieving a bound of $\tilde O(N d / \eps^5)$ 
by combining their approach and the techniques from~\cite{WangMMR15}.

Consider the following packing SDP~\eqref{eqn:primal-standard} and its dual covering 
SDP~\eqref{eqn:dual-standard} with parameters $(\nu,\eps,\rho)$:

\begin{lp}
\label{eqn:primal-standard}
\maxi{\mone^\top w}
\st \con{w_i \ge 0, \sum_{i=1}^N w_i A_i \preceq I}
\end{lp}
\begin{lp}
\label{eqn:dual-standard}
\mini{\tr(M') + \normone{y'}}
\st \con{\rho X_i^\top M' X_i + (1-\eps)N y_i \ge 1, M' \succeq 0, y' \ge 0 \;,}  
\end{lp}
where each $A_i \in \R^{(d+N)\times(d+N)}$ is a PSD matrix given by\footnote{Recall that $X_i \in \R^{d \times 1}$ is the $i$-th sample, and $e_i \in \R^{N \times 1}$ is the $i$-th standard basis vector.}
\begin{align*}
A_i & = \left[ \begin{array}{cc}
\rho (X_i-\nu) (X_i-\nu)^\top & 0 \\ 0 & (1-\eps)N \cdot e_i e_i^\top
\end{array} \right] \; .
\end{align*}

We will first show that the solutions of \eqref{eqn:primal-standard} and \eqref{eqn:dual-standard} 
are closely related to solutions of our original SDPs \eqref{eqn:primal-sdp} and \eqref{eqn:dual-sdp}.
Formally, the following lemma shows that if we (approximately) solve the packing/covering 
SDPs~\eqref{eqn:primal-standard},~\eqref{eqn:dual-standard} for some value of $\rho > 0$ 
and the resulting objective values are close to $1$, then we can translate these solutions back 
to obtain solutions for SDPs~\eqref{eqn:primal-sdp}~\eqref{eqn:dual-sdp} with objective value roughly $1/\rho$.

\begin{lemma}
\label{lem:primal-dual-convert-back}
Fix $\nu \in \R^d, 0 < \eps < 1/3$, and $\rho > 0$.
If we have a solution $w'$ of SDP~\eqref{eqn:primal-standard} with parameters $(\nu, \eps)$ such that $\normone{w'} \ge 1 - \frac{\eps}{10}$, then we can construct a solution $w$ of SDP~\eqref{eqn:primal-sdp} with parameters $(\nu, 2\eps)$ whose objective value is at most $\frac{1}{\rho (1-\eps/10)}$.
If we have a solution $(M', y')$ of SDP~\eqref{eqn:dual-standard} with parameters $(\nu, \eps)$ such that $\tr(M') + \normone{y'} \le 1$, then we can construct a solution $M$ of SDP~\eqref{eqn:dual-sdp} with parameters $(\nu, \eps)$ whose objective value is at least $1/\rho$.
\end{lemma}
\begin{proof}
We first construct a solution $w$ to SDP~\eqref{eqn:primal-sdp} with parameters $(\nu, 2\eps)$ given $w'$.
Let $w = \frac{w'}{\normone{w'}}$.
Since $\normone{w'} \ge 1 - \frac{\eps}{10}$, we know that $w \in \Delta_{N, 2\eps}$ is feasible for SDP~\eqref{eqn:primal-sdp}.
SDP~\eqref{eqn:primal-standard} guarantees that $\rho \sum_i w'_i (X_i-\nu) (X_i-\nu)^\top \preceq I$, so the objective value of SDP~\eqref{eqn:primal-sdp} at $w$ satisfies $\lambda_{\max}\left(\sum_i w_i (X_i-\nu) (X_i-\nu)^\top\right) \le \frac{1}{\rho (1-\eps/10)}$.

Next, we will construct a solution $M$ for the original dual SDP~\eqref{eqn:dual-sdp} with parameters $(\nu, \eps)$ given $(M', y')$.  
We will work with the following SDP that is equivalent to the dual SDP~\eqref{eqn:dual-sdp}:
\begin{lp*}
\maxi{z - \frac{\sum_{i=1}^N y_i}{(1-\eps)N}}
\st \con{M \succeq 0, \tr(M) \le 1, y \ge 0, (X_i - \nu)^\top M (X_i - \nu) + y_i \ge z}
\end{lp*}

Given a dual solution $M \in \R^{d \times d}$, it is easy to find the optimal values of $(y,z)$ variables: 
$z$ should be at the $(1-\eps)$-th quantile for $(X_i - \nu)^\top M (X_i - \nu)$ 
and $y_i = \max\{z-(X_i - \nu)^\top M (X_i - \nu), 0\}$.
Under these choices, we recover the objective value of SDP~\eqref{eqn:dual-sdp}, which is the mean of the smallest $(1-\eps)N$ entries of $\left((X_i - \nu)^\top M (X_i - \nu)\right)_{i=1}^N$.

Let $M = \frac{M'}{\tr(M')}$, $y = \frac{(1-\eps)N}{\rho \tr(M')} y'$, and $z = \frac{1}{\rho \tr(M')}$.
Note that $y$ is well-defined, because we always have $M' \neq \bf{0}$, otherwise the objective value is at least $1/(1-\eps) > 1$.
Note that $(M, y, z)$ is a feasible solution to SDP~\eqref{eqn:dual-sdp}: 
by the definition of $(M, y, z)$, the constraint $\rho X_i^\top M' X_i + (1-\eps)N y'_i \ge 1$ translates to $\rho \tr(M') X_i^\top M X_i + \rho \tr(M') y_i \ge 1 = \rho \tr(M') z$, which is exactly $X_i^\top M X_i + y_i \ge z$.
The objective value of $(M, y, z)$ is $z - \frac{\normone{y}}{(1-\eps)N} = \frac{1 - \normone{y'}}{\rho \tr(M')} = 1/\rho$.
\end{proof}

We will use binary search to find a suitable $\rho > 0$, solve the packing/covering SDPs approximately, 
and then translate the solutions back using Lemma~\ref{lem:primal-dual-convert-back}.
The translated solutions will satisfy the conditions of Proposition~\ref{prop:approx-sdp}. 
To make sure a suitable $\rho$ exists, we use the following lemma which 
shows the optimal value of SDPs~\eqref{eqn:primal-standard},~\eqref{eqn:dual-standard} is continuous and monotone in $\rho$.
\begin{lemma}
\label{lem:opt-rho-monotone}
Fix $\nu \in \R^d$ and $0 < \eps < 1/3$.
Let $\OPT_{\rho} = \OPT_{\rho,\nu,\eps}$ be the optimal value of SDPs~\eqref{eqn:primal-standard},~\eqref{eqn:dual-standard} 
with parameters $\nu$, $\eps$, and $\rho > 0$.
Then, $\OPT_\rho$ is continuous and non-increasing in $\rho$.
Moreover, for $\rho^* = 1/\OPT_{\nu,\eps}$, we have $\OPT_{\rho^*} = 1$.
\end{lemma}
\begin{proof}
To prove $\OPT_\rho$ is continuous and non-increasing in $\rho > 0$, it is sufficient to show that $\OPT_{\rho_1} \ge \OPT_{\rho_2} \ge (1-\gamma)\OPT_{\rho_1}$ for any $0 < \gamma < 1$ and $\rho_1 = (1-\gamma) \rho_2$.
Let $w_1$, $w_2$ be optimal solutions that achieve $\OPT_{\rho_1}$ and $\OPT_{\rho_2}$ respectively.
Because $\rho_1 < \rho_2$, $w_2$ is feasible for SDP~\eqref{eqn:primal-standard} with parameter $\rho_1$, and therefore $\OPT_{\rho_1} \ge \normone{w_2} =  \OPT_{\rho_2}$.
Similarly, $(1-\gamma)w_1$ is feasible for SDP~\eqref{eqn:primal-standard} with parameter $\rho_2$, because $(1-\gamma)w_1 \rho_2 = w_1 \rho_1$, and thus $\OPT_{\rho_2} \ge \normone{(1-\gamma)w_1} = (1-\gamma)\OPT_{\rho_1}$.

To prove $\OPT_{\rho^*}=1$, we focus on the original primal SDP~\eqref{eqn:primal-sdp} and the packing SDP~\eqref{eqn:primal-standard}.
We first prove $\OPT_{\rho^*} \ge 1$.
Consider the optimal solution $w \in \R^N$ of SDP~\eqref{eqn:primal-sdp} with parameters $(\nu, \eps)$.
We can verify that $w$ is feasible for SDP~\eqref{eqn:primal-standard} with parameters $(\rho^*, \nu, \eps)$:
The constraint on the bottom-right block of SDP~\eqref{eqn:primal-standard} states that $w_i \le \frac{1}{(1-\eps)N}$ for all $i \in [N]$, and the constraint on the top-left block is equivalent to $\sum_i w_i X_i X_i^\top \le \frac{1}{\rho^*} I = \OPT_{\nu,\eps} \cdot I$.
Now assume if $\OPT_{\rho^*} > 1$.
Let $w'$ be the optimal solution to SDP~\eqref{eqn:primal-standard} with $\normone{w'} > 1$.
This leads to a contradiction, because $w = w' / \normone{w'}$ is feasible for 
SDP~\eqref{eqn:primal-sdp}, but its objective value is $\frac{1}{\rho^* \normone{w'}} < \OPT_{\nu,\eps}$.
\end{proof}

Now we are ready to prove Proposition~\ref{prop:approx-sdp} by putting Lemmas~\ref{lem:opt-rho-monotone}~and~\ref{lem:primal-dual-convert-back} together. 

\begin{proof}[\bf Proof of Proposition~\ref{prop:approx-sdp}]
Fix $\nu \in \R^d$ and $0 < \eps < 1/3$.
We first prove that it is sufficient to find some $\rho > 0$, such that we can compute both
\begin{enumerate}
\item[(i)] a solution $w \in \R^N$ to SDP~\eqref{eqn:primal-sdp} with parameters $(\nu, 2\eps)$, whose objective value is $\frac{1}{\rho(1-\eps/10)}$;
\item[(ii)] a solution $M \in \R^{d \times d}$ to SDP~\eqref{eqn:dual-sdp} with parameters $(\nu, \eps)$, whose objective value is $\frac{1}{\rho}$.
\end{enumerate}

This holds for the following reason:
Let $\ALG_P = \frac{1}{\rho(1-\eps/10)}$ and $\ALG_D = \frac{1}{\rho}$.
Then we must have that either $\ALG_P \le 1 + c_4 \beta^2$ or $\ALG_D \ge 1 + \frac{9}{10} c_4 \beta^2$.
Assuming $\ALG_P > 1 + c_4 \beta^2$ and $\ALG_D < 1 + \frac{9}{10} c_4 \beta^2$ leads to the following contradiction:
$$
1-\frac{\eps}{10} = \frac{1/\rho}{1/(\rho(1-\eps/10))} = \frac{\ALG_D}{\ALG_P} < 1 - \frac{0.1 c_4 \beta^2}{1+c_4\beta^2} \le 1 - \min(\frac{1}{20}, \frac{c_4}{20} \beta^2) \le 1-\frac{\eps}{10} \;.
$$
Moreover, because $\ALG_P$ is the value of a solution to SDP~\eqref{eqn:primal-sdp} with parameters $(\nu, 2\eps)$, we know that $\ALG_D = (1-\frac{\eps}{10}) \ALG_P \ge \OPT_{\nu, 2\eps}$, as needed. 

We will define a target interval $[\rho_1, \rho_2]$, such that solving 
SDPs~\eqref{eqn:primal-standard},~\eqref{eqn:dual-standard} 
for any parameters ($\rho \in [\rho_1, \rho_2], \nu, \eps)$ 
will allow us to compute a pair of solutions $w'$ and $(M', y')$ such that:
\begin{enumerate}
\item[(i)] $w'$ is a solution to packing SDP~\eqref{eqn:primal-standard} with $\normone{w'} \ge 1 - \frac{\eps}{10}$; and
\item[(ii)] $(M', y')$ is a solution to covering SDP~\eqref{eqn:dual-standard} with $\tr(M') + \normone{y'} \le 1$.
\end{enumerate}
Then, by Lemma~\ref{lem:primal-dual-convert-back}, we can convert these solutions 
to solutions of SDP~\eqref{eqn:primal-sdp}~and~\eqref{eqn:dual-sdp} with values 
$\ALG_P$ and $\ALG_D$.

We can first solve SDP~\eqref{eqn:primal-standard} with $\rho = 1$, and check 
if the solution satisfies $\normone{w'} \ge 1 - \frac{\eps}{10}$.
If so, we use Lemma~\ref{lem:primal-dual-convert-back} to convert $w'$ back to a solution 
of SDP~\eqref{eqn:primal-sdp} whose objective value is $\frac{1}{1-\eps/10} \le 1 + c_4 \beta^2$ 
and we are done.
For the rest of the proof, we assume $\OPT_{\rho = 1} < 1 - \frac{2\eps}{30}$. 

Fix any $\rho_1 \in \{\rho: \OPT_{\rho} = 1 - \frac{\eps}{30}\}$ and $\rho_2 \in \{\rho: \OPT_{\rho} = 1 - \frac{2\eps}{30} \}$.
Note that they are well-defined because $\OPT_\rho$ is continuous, $\OPT_{\rho^*} = 1$, and $\OPT_{\rho=1} < 1 - \frac{2\eps}{30}$.
For any $\rho \in [\rho_1, \rho_2]$, by monotonicity, we must have $\OPT_{\rho} \in [1-\frac{2\eps}{30}, 1-\frac{\eps}{30}]$.
Therefore, if we can solve the packing/covering SDPs~\eqref{eqn:primal-standard},~\eqref{eqn:dual-standard} approximately up to a multiplicative factor of $(1 \pm O(\eps))$, we can find a primal solution $w'$ with $\normone{w'} \ge 1-\frac{\eps}{10}$, as well as a dual solution $(M', y')$ with $\tr(M') + \normone{y'} \le 1$.

It remains to show that we can find a suitable $\rho$ and solve the SDPs~\eqref{eqn:primal-standard}~\eqref{eqn:dual-standard} in time $\tilde O(N d / \eps^6)$.
We can find $\rho \in [\rho_1, \rho_2]$ using binary search: if $\normone{w'} < 1 - \frac{\eps}{10}$ we will decrease $\rho$, and if $\tr(M') + \normone{y'} > 1$ we will increase $\rho$.

We first show that the binary search takes $O(\log(d/\eps))$ steps.
Observe that by monotonicity, $0 < \rho^* \le \rho_1 \le \rho_2 < 1$.
When $\normtwo{\nu - \mus} \le O(\eps \sqrt{d})$, Lemma~\ref{lem:ub-lb-opt} 
implies that $\OPT_{\nu,\eps} \le O(d)$ and hence $\rho^* = 1/\OPT_{\nu,\eps} \ge \Omega(1/d)$.
If $\rho_1 \ge (1-\frac{\eps}{30})\rho_2$, then by the same argument as 
in the proof of Lemma~\ref{lem:opt-rho-monotone}, we must have 
$\OPT_{\rho_2} \ge (1-\frac{\eps}{30}) \OPT_{\rho_1} = (1-\frac{\eps}{30})^2 > \OPT_{\rho_2}$, a contradiction.
Therefore, the interval $[\rho_1, \rho_2]$ has length at least $\rho_2 - \rho_1 \ge \frac{\eps}{30} \rho_2 \ge \Omega(\eps \rho^*) = \Omega(\eps/d)$.
In summary, we start with an interval of length less than $1$, and the target interval has length 
at least $\Omega(\eps/d)$, so binary search needs at most $O(\log(d/\eps))$ steps.

Finally, we bound from above the running time of the algorithm in this proposition.
In each step of the binary search, we solve packing/covering 
SDPs~\eqref{eqn:primal-standard},~\eqref{eqn:dual-standard} for some $\rho$.
We solve these SDPs to precision $(1 \pm O(\eps))$ as required in this proof, 
which takes time $\tilde O(N d /\eps^6)$, by Corollary~\ref{cor:sdp-runtime} from Section~\ref{sec:sdp-solver}.
We repeat every use of Corollary~\ref{cor:sdp-runtime} $O(\log \log(d/\eps))$ times, 
so that the failure probability is at most $1/10$ when we take a union bound over all $O(\log(d/\eps))$ iterations.
Eventually, when we have a suitable $\rho$, we can convert the solution back to solutions 
for SDPs~\eqref{eqn:primal-sdp}~\eqref{eqn:dual-sdp} using Lemma~\ref{lem:primal-dual-convert-back}.
Therefore, the total running time is 
\[
O(\log(d/\eps)) \cdot \tilde O((N d)/\eps^6) \cdot O(\log \log(d/\eps)) = \tilde O(N d /\eps^6) \;. \qedhere
\]
\end{proof}


\subsection{Positive SDP Solvers}
\label{sec:sdp-solver}
In this subsection, we show how to solve packing/covering SDPs~\eqref{eqn:primal-standard},~\eqref{eqn:dual-standard} 
in time $\tilde O(N d / \eps^6)$. It is known that positive (i.e., packing/covering) SDPs can be solved 
in nearly-linear time and poly-logarithmic number of iterations~\cite{JainY11, AllenLO16, PengTZ16}.
Because SDPs~\eqref{eqn:primal-standard},~\eqref{eqn:dual-standard} are packing/covering SDPs, 
we can apply the positive SDP solvers in~\cite{PengTZ16} directly (Corollary~\ref{cor:sdp-runtime}).

\begin{lemma}[Positive SDP Solver,~\cite{PengTZ16}]
\label{lem:ptz16}
Let $A_1, \ldots, A_n$ be $m \times m$ PSD matrices given in factorized form $A_i = C_i C_i^\top$.
Consider the following pair of packing and covering SDPs:
\begin{alignat*}{3}
\max_{x \ge 0}     & \; \mone^\top x && \quad \textup{ s.t. } \sum_{i=1}^n x_i A_i \preceq I \; . \\
\max_{Y \succeq 0} & \; \tr(Y)       && \quad \textup{ s.t. } A_i \bullet Y \ge 1, \forall i \; .
\end{alignat*}
We can compute, with probability at least $9/10$, a feasible solution $x$ to the packing SDP with $\mone^\top x \ge (1-\eps) \OPT$,
  and together a feasible solution $Y$ to the covering SDP with $\tr(Y) \le (1+\eps) \OPT$ in time $\tilde O((n + m + q) / \eps^{6})$, where $q$ is the total number of non-zero entries in the $C_i$'s.
\end{lemma}

An application of the above lemma yields the following corollary:

\begin{corollary}
\label{cor:sdp-runtime}
Fix $\nu \in \R^d$, $0 < \eps < 1/3$, and $0 < \rho \le 1$.
We can compute in $\tilde O(N d / \eps^6)$ time, with probability at least $9/10$,
\begin{enumerate}
\item a $(1+O(\eps))$-approximate solution $w'$ for the packing SDP~\eqref{eqn:primal-standard} with parameters $(\nu, \eps, \rho)$; and
\item a $(1-O(\eps))$-approximate solution $(M', y')$ for the covering SDP~\eqref{eqn:dual-standard} with parameters $(\nu, \eps, \rho)$.
\end{enumerate}
\end{corollary}
\begin{proof}
The input matrices $A_i \in \R^{(d+N) \times (d+N)}$ in the SDPs can be factorized as $A_i = C_i C_i^\top$, where
\begin{align*}
C_i & = \left[ \begin{array}{cc}
\sqrt{\rho} (X_i-\nu) \quad 0_{d \times (d-1)} & 0_{N \times d} \\ 0_{d \times N} & \sqrt{(1-\eps)N} \cdot e_i e_i^\top
\end{array} \right] \; .
\end{align*}
The total number of non-zeros in all $C_i$'s is $q = N (d + 1) = O(Nd)$, 
so by Lemma~\ref{lem:ptz16}, we can solve SDPs~\eqref{eqn:primal-standard},~\eqref{eqn:dual-standard} 
in time $\tilde O(Nd / \eps^6)$ with probability $9/10$.
Note that the dual solution should be maintained implicitly to avoid writing down 
an $(N+d) \times (N+d)$ matrix: the top-left block of the dual solution $Y$ is $M'$, and the diagonals of the bottom-right block is $y'$.
\end{proof}

%
%


\section{Robust Mean Estimation under Second Moment Assumptions}
\label{sec:bounded}

In this section, we use the algorithmic ideas from Section~\ref{sec:mean} 
to establish Theorem~\ref{thm:cov-mean}. The algorithm in this case is similar
to the one for the sub-gaussian case with some important 
differences, due to the different concentration properties in the two settings.

Note that it suffices to prove Theorem~\ref{thm:cov-mean} under the assumption that $\sigma = 1$, i.e.,
the covariance satisfies $\Sigma \preceq I$.
This is without loss of generality: Given a distribution $D$ with $\Sigma \preceq \sigma^2 I$, 
we can first divide every sample by $\sigma$, run the algorithm to learn the mean, 
and multiply the output by $\sigma$.

We will require a set of conditions to hold for the good samples (similar to Conditions~\eqref{eqn:good-sample-moments}~and~\eqref{eqn:good-pruning} for sub-gaussian distributions in Section~\ref{thm:cov-mean}).
We would like the set of good samples to have bounded variance in all directions.
However, because we did not make any assumption on the higher moments, 
  it may be possible for a few \emph{good} samples to affect the empirical covariance too much.
Fortunately, such samples have small probability and they do not contribute much to the mean, so we can remove them in the preprocessing step (see Remark~\ref{rem:naive-pruning}).

Recall that $N = |G^\star|$, and $\Delta_{N,\eps}$ is the set
$\Delta_{N,\eps} = \{w \in \R^N : \sum_i w_i = 1 \text{ and } 0 \le w_i \le \frac{1}{(1-\eps) N} \text{ for all } i\}$.
We require the following condition to hold:
for all $w \in \Delta_{N, 3\eps}$,
\begin{align}
\label{eqn:conditions-cov}
\begin{split}
\normtwo{\sum_{i \in G^\star} w_i (X_i - \mus)} \le \delta \; , \quad
\normtwo{\sum_{i \in G^\star} w_i (X_i - \mus) (X_i - \mus)^\top} \le \delta_2 \; , \\
\forall i\in G^\star, \; \normtwo{X_i-\mus} \le O(\sqrt{d/\eps}) \; .
\end{split}
\end{align}
where $\delta = c_1 \sqrt{\eps}$ and $\delta_2 = c_1$ for some universal constants $c_1$.
It follows from Lemma~A.18 of~\cite{DKK+17} that these conditions will be satisfied with high constant probability after $N = \Omega((d\log d)/\eps)$ samples.

The high-level approach is the same as in learning the mean of sub-gaussian distributions: we maintain $\nu \in \R^d$ as our current guess for the unknown mean $\mus$, and try to move $\nu$ closer to $\mus$ by solving the dual SDP~\eqref{eqn:dual-sdp}; eventually $\nu$ will be close enough, and the primal SDP~\eqref{eqn:primal-sdp} can provide good weights $w$ so that we can output the weighted empirical mean $\hat \mu_w$.

The algorithm will be almost identical to Algorithm~\ref{alg:pc-sdp}.
The only difference is in the ``if'' statement, where we need a different threshold to decide if the current primal SDP solution is good (or equivalently, whether our guess of $\nu$ is close enough to $\mus$).

We are now ready to present our algorithm (Algorithm~\ref{alg:sdp-boundedcov}) to robustly estimate the mean of bounded covariance distributions.
Algorithm~\ref{alg:sdp-boundedcov} is almost identical to Algorithm~\ref{alg:pc-sdp}.
The differences are highlighted in bold font.
In the {\bf if} statement, we need a different threshold to decide whether the current primal SDP solution is good (or equivalently, whether our guess of $\nu$ is close enough to $\mus$).
We use Proposition~\ref{prop:apx-solve-cov} to solve the SDPs.
In addition, we need to replace Lemmas~\ref{lem:wrong-mean-primal-nosol}~and~\ref{lem:good-dual-better-nu} with Lemmas~\ref{lem:cov-good-primal}~and~\ref{lem:cov-good-dual}, but this merely changes our analysis and has no impact on the algorithm.

\begin{algorithm}[h]
  \caption{Robust Mean Estimation for Bounded Covariance Distributions}
  \label{alg:sdp-boundedcov}
  \SetKwInOut{Input}{Input}
  \SetKwInOut{Output}{Output}
  \Input{An $\eps$-corrupted set of $N$ samples $\{X_i\}_{i=1}^N$ on $\R^d$ with $N = \tilde \Omega(d/\eps^2)$ and $\eps < 1/3$ (after the preprocessing step described in Remark~\ref{rem:naive-pruning}).}
  \Output{A vector $\hat \mu \in \R^d$ such that, with probability $9/10$, $\normtwo{\hat \mu - \mus} \le {\color{red} \bm{O(\sqrt{\eps})}}$.}
  Let $\nu \in \R^d$ be the coordinate-wise median of $\{X_i\}_{i=1}^N$\;
  \For{$i = 1$ {\bf to} $O(\log d)$}{
   Use {\color{red} \bf Proposition~\ref{prop:apx-solve-cov}} to compute either \\
   {\em (i)} A good solution $w \in \R^N$ for the primal SDP~\eqref{eqn:primal-sdp} with parameters $\nu$ and $2\eps$; or \\
   {\em (ii)} A good solution $M \in \R^{d \times d}$ for the dual SDP~\eqref{eqn:dual-sdp} with parameters $\nu$ and $\eps$\;
   \eIf{the objective value of $w$ in SDP~\eqref{eqn:primal-sdp} is at most {\color{red} $\bm{c_4}$}}{
     \Return{the weighted empirical mean $\hat \mu_w = \sum_{i=1}^N w_i X_i$} (Lemma~\ref{lem:cov-good-primal})\;
   }{
     Move $\nu$ closer to $\mus$ using the top eigenvector of $M$ (Lemma~\ref{lem:cov-good-dual}). }
  }
\end{algorithm}

\begin{remark} \label{rem:naive-pruning}
{\em We wish to throw away samples that are too far from $\mus$.
However, we do not know $\mus$, so we run the following preprocessing step.
We start with an $\eps$-corrupted set of $2N$ samples drawn from $D$ (the true distribution) and partition them into two sets of $N$ samples, $S_1$ and $S_2$.
We first compute the coordinate-wise median $\tilde \mu$ of $S_1$.
Notice that $\normtwo{\tilde \mu - \mus} \le O(\eps\sqrt{d})$ with high probability.
We will use $\tilde \mu$ to throw away samples that are too far from $\mus$.

Let $\mathcal{B}$ be a ball of radius $O(\sqrt{d/\eps})$ around $\tilde \mu$.
Let $G^\star_2$ denote the original set of uncorrupted samples corresponds to $S_2$. 
By the bounded-covariance assumption, we know that with high probability, $(1-O(\eps))$-fraction of the samples in $G^\star_2$ are in $\mathcal{B}$.
Thus, if we change all samples in $S_2 \setminus \mathcal{B}$ to $\tilde \mu$, the resulting set $S'$ is an $O(\eps)$-corrupted set of samples drawn from $D$,
  and all samples in $S'$ are not too far from $\mus$.
Therefore we can use $S'$ as the input to Algorithm~\ref{alg:sdp-boundedcov}.
In the rest of this section, we abuse notation and use $\eps$ to denote the fraction of corrupted samples in $S'$.}
\end{remark}

\medskip

\noindent {\bf Notation.} In this section, we use $c_1, \ldots, c_6$ to denote universal constants.
They can be chosen in a way that is similar to how we set constants for Section~\ref{sec:mean} in Appendix~\ref{apx:const}.




\subsection{Optimal Values of the SDPs}
The following lemma is similar to Lemma~\ref{lem:ub-lb-opt}.
Specifically, Lemma~\ref{lem:ub-lb-opt-cov} shows that when 
$\normtwo{\mus-\nu} \ge c_2 \beta$, $\OPT$ is approximately $\normtwo{\mus-\nu}^2$.
The difference is that {\it (i)} in Section~\ref{sec:mean}, $\OPT$ is roughly $1 + \normtwo{\mus-\nu}^2$, because we know the true covariance matrix is $I$, but in this section we only know the second moment matrix is bounded; and {\it (ii)} we need $\normtwo{\mus-\nu} \ge c_2 \beta$ in both settings, but in this section $\beta = \Theta(\sqrt{\delta^2/\eps}) = 1$ (rather than $\beta = \sqrt{\eps\ln(1/\eps)}$ as in Section~\ref{sec:mean}) because the values of $\delta$ from the concentration bounds is different.

\begin{lemma}
\label{lem:ub-lb-opt-cov}
Fix $0 < \eps < 1/3$ and $\nu \in \R^d$.
Let $\delta = c_1 \sqrt{\eps}$, $\delta_2 = c_1$, and $\beta = 1$.
Let $\{X_i\}_{i=1}^N$ be a set of $\eps$-corrupted samples drawn from a distribution on $\R^d$ with $\Sigma \preceq I$, where $N = \Omega((d \log d)/\eps)$. 
Let $\OPT_{\nu,\eps}$ denote the optimal value of the SDPs~\eqref{eqn:primal-sdp},~\eqref{eqn:dual-sdp} with parameters $(\nu,\eps)$.
Let $r = \normtwo{\nu - \mus}$. Then, we have
\[
r^2 - 2 \delta r \le \OPT_{\nu,2\eps} \le \OPT_{\nu,\eps} \le \delta_2 + r^2 + 2\delta r \; .
\]
In particular, when $\normtwo{\mus - \nu} \ge c_2 \beta$, we have
\[
0.9 r^2 \le \OPT_{\nu,2\eps} \le \OPT_{\nu,\eps} \le 1.1 r^2 \; .
\]
\end{lemma}
\begin{proof}
We take the same feasible primal/dual solutions as in the proof of Lemma~\ref{lem:ub-lb-opt}.
We get different upper/lower bounds because we use Conditions~\eqref{eqn:conditions-cov} in this section.

Consider a feasible primal solution $w$ with $w_i = \frac{1}{|G|}$ for all $i \in G$ and $w_i = 0$ otherwise.
\begin{align*}
\OPT_{\nu,\eps} & \le \max_{y \in \R^d, \normtwo{y}=1} \left(\sum_{i\in G} w_i \inner{X_i - \mus, y}^2 + \inner{\mus - \nu, y}^2 + 2\inner{\sum_{i\in G} w_i (X_i - \mus), y}\inner{\mus - \nu, y}\right) \\
  & \le \delta_2 + \normtwo{\mus - \nu}^2 + 2\delta \normtwo{\mus - \nu} \; .
\end{align*}

One feasible dual solution is $M = yy^\top$ where $y = \frac{\mus - \nu}{\normtwo{\mus - \nu}}$.
Let $S$ denote the $(1-2\eps)N$ good samples with smallest $(X_i - \nu)^\top M (X_i - \nu)$.
Let $w'_i = \frac{1}{(1-2\eps)N}$ for all $i \in S$ and $w'_i = 0$ otherwise.
\begin{align*}
\OPT_{\nu,\eps} 
  & \ge \sum_{i \in G} w'_i \inner{X_i - \mus, y}^2 + w'_G \normtwo{\mus - \nu}^2 + 2 \sum_{i \in G} w'_i \inner{X_i - \mus, y} \normtwo{\mus - \nu} \\
  & \ge \normtwo{\mus - \nu}^2 - 2 \delta \normtwo{\mus - \nu} \; . 
\end{align*}

The same upper/lower bounds hold for $\OPT_{\nu,2\eps}$ as well.
\end{proof}

\subsection{When Primal SDP Has Good Solutions}
We prove that if the weighted empirical mean is far away from the true mean, then the value of SDP~\eqref{eqn:primal-sdp} must be large.

The next lemma is similar to Lemma~\ref{lem:wrong-mean-primal-nosol}.
The same intuition still holds: if $\eps$-fraction of the samples distort the mean by $\Omega(\delta)$, 
then they must introduce $\Omega(\delta^2/\eps)$ error to the second moment matrix.
Note that because the true covariance matrix is no longer $I$, we cannot say 
anything about the contribution of the good samples.

\begin{lemma}
\label{lem:cov-good-primal}
Fix $0 < \eps < 1/3$.
Let $\delta = c_1 \sqrt{\eps}$, $\delta_2 = c_1$, and $\beta = 1$.
For all $w \in \Delta_{N, 2\eps}$, if $\normtwo{\hat \mu_w - \mus} \ge c_3 \delta$, then for all $\nu \in \R^d$, $\lambda_{\max}\left(\sum_{i=1}^N w_i (X_i - \nu)(X_i - \nu)^\top\right) \ge c_4 \beta^2$.
\end{lemma}
\begin{proof}
By Lemma~\ref{lem:ub-lb-opt-cov}, we know that if $\normtwo{\mus - \nu} \ge c_5 \beta$, we have $\OPT_{\nu,2\eps} \ge 0.9 \normtwo{\mus - \nu}^2 \ge c_4 \beta^2$. 
Therefore, we can assume that $\normtwo{\mus - \nu} < c_5 \beta$.

Let $w \in \Delta_{N, 2\eps}$ denote the optimal primal solution.
For $y = (\hat \mu_w - \mus) / \normtwo{\hat \mu_w - \mus}$, we have
\begin{align*}
\abs{\sum_{i\in B} w_i \inner{X_i - \nu, y}}
  & \ge \normtwo{\hat \mu_w - \mus} - \delta - 2\eps c_5 \beta \ge c_6 \delta \; . 
\end{align*}

By the Cauchy-Schwarz inequality and the fact that $w_B \le 2\eps$,
we get that $\sum_{i\in B} w_i \inner{X_i - \nu, y}^2 \ge \frac{c_6^2}{2} (\delta^2/\eps)$.
We conclude the proof by observing that
\begin{align*}
\lambda_{\max} \left(w_i (X_i - \nu)(X_i - \nu)^\top \right)
  & \ge \sum_{i=1}^N w_i \inner{X_i - \nu, y}^2 \ge \sum_{i\in B} w_i \inner{X_i - \nu, y}^2 
  \ge \frac{c_6^2 \delta^2}{2\eps} \ge c_4 \beta^2 \; . \qedhere 
\end{align*}
\end{proof}

In summary, if we can find a solution $w \in \R^N$ to the primal SDP~\eqref{eqn:primal-sdp} whose objective value is $O(\beta^2)$, then we are done because Lemma~\ref{lem:cov-good-primal} guarantees that (no matter what $\nu$ is) the weighted empirical mean $\hat \mu_w$ is close to the true mean.

\subsection{When Primal SDP Has No Good Solutions}

We show that when the primal SDP has no good solutions, we can solve the dual (approximately) and the dual will allow us to move $\nu$ closer to $\mus$ by a constant factor.
The next lemma is similar to Lemma~\ref{lem:good-dual-better-nu}.
The first part of the argument changes slightly because we are using Condition~\eqref{eqn:conditions-cov}, 
while the second part (the geometric argument) is identical to that of Lemma~\ref{lem:good-dual-better-nu}.

\begin{lemma}
\label{lem:cov-good-dual}
Fix $0 < \eps < 1/3$ and $\nu \in \R^d$.
Assume $M \in \R^{d \times d}$ is a solution to dual SDP~\eqref{eqn:dual-sdp} with parameters $(\nu,\eps)$, and the objective value of $M$ is at least $\max\left(0.9 c_4 \beta^2, 0.95 \, \OPT_{\nu,2\eps}\right)$.
Then, we can efficiently find a vector $\nu'$ such that $\normtwo{\nu' - \mus} \le \frac{3}{4} \normtwo{\nu - \mus}$.
\end{lemma}
\begin{proof}
By Lemma~\ref{lem:ub-lb-opt-cov}, we know that $\OPT_{\nu,\eps} \ge 0.9 c_4 \beta^2$ implies that $\normtwo{\mus - \nu} \ge c_2 \beta$ and thus $0.95 \, \OPT_{\nu,2\eps} \ge 0.85 \normtwo{\nu - \mus}^2$.
The dual objective is the mean of the smallest $(1-\eps)$-fraction of the entries $X_i^\top M X_i$.
Because one way to choose $(1-\eps)$-fraction is to focus on the good samples,
$
\frac{1}{|G|} \sum_{i \in G} (X_i - \nu)^\top M (X_i - \nu) \ge 0.85 \normtwo{\mus-\nu}^2. 
$

We know that $M \succeq 0$, $\tr(M) = 1$.
Without loss of generality, we can assume $M$ is symmetric.
By Condition~\eqref{eqn:conditions-cov}, we can prove 
$\inner{M, (\mus - \nu)(\mus - \nu)^\top} \ge \frac{3}{4} \normtwo{\mus-\nu}^2$ as follows.
\begin{align*}
0.85 \normtwo{\mus - \nu}^2
 & \le \frac{1}{|G|}\sum_{i\in G} \inner{M, (X_i - \mus)(X_i - \mus)^\top + 2(X_i - \mus)(\mus-\nu) + (\mus-\nu)(\mus-\nu)^\top} \\
 & \le \delta_2 + 2\delta \normtwo{\mus-\nu} + \inner{M, (\mus-\nu)(\mus-\nu)^\top} \\
 & \le 0.1 \normtwo{\mus-\nu}^2 + \inner{M, (\mus-\nu)(\mus-\nu)^\top} \; .
\end{align*}
Therefore, we have a matrix whose inner product with $(\mus-\nu)(\mus-\nu)^\top$ is approximately maximized, this implies that the top eigenvector of $M$ aligns with $(\nu - \mus)$.
We omit the rest of the proof because the geometric analysis is identical to that of Lemma~\ref{lem:good-dual-better-nu}.
\end{proof}

\subsection{Proof of Theorem~\ref{thm:cov-mean}}
In this section we prove Theorem~\ref{thm:cov-mean} (Correctness and Runtime of Algorithm~\ref{alg:sdp-boundedcov}).
By combining Lemmas~\ref{lem:cov-good-primal}~and~\ref{lem:cov-good-dual}, we can make progress by either finding a solution to the primal SDP~\eqref{eqn:primal-sdp} with objective value at most $c_4 \beta^2$, 
or finding an approximately optimal solution to the dual SDP~\eqref{eqn:dual-sdp} 
whose objective value is at least $0.9 c_4 \beta^2$.
This next proposition shows that this can be done 
in time $\tilde O(N d) / \poly(\eps)$.

\begin{proposition}
\label{prop:apx-solve-cov}
Fix $0 < \eps < 1/3$, and $\nu \in \R^d$ with $\normtwo{\nu - \mus} \le O(\sqrt{d/\eps})$.
We can compute in time $\tilde O(N d)/ \eps^6)$, with probability at least $9/10$, either (i) a solution $w$ for primal SDP~\eqref{eqn:primal-sdp} with parameters $(\nu, 2\eps)$ whose objective value is at most $c_4 \beta^2$; or (ii) a solution $M$ for dual SDP~\eqref{eqn:dual-sdp} with parameter $(\nu, \eps)$ whose objective is at least $\max\left(0.9 c_4 \beta^2, (1-\frac{\eps}{10}) \OPT_{\nu,2\eps}\right)$.
\end{proposition}

We omit the proof of Proposition~\ref{prop:apx-solve-cov} because its proof is almost identical to the proof of Proposition~\ref{prop:approx-sdp}.
The only difference is that the ratio between the objective values of the desired primal/dual solutions is now $\frac{0.9 c_4 \beta^2}{c_4 \beta^2} = 0.9$, instead of $\frac{1+0.9c_4\beta^2}{1+c_4\beta^2}$ as in Proposition~\ref{prop:approx-sdp}.
The problem of computing a desired pair of solutions becomes easier since the gap is larger.

Theorem~\ref{thm:cov-mean} follows directly from Lemmas~\ref{lem:cov-good-primal},~\ref{lem:cov-good-dual},~and~Proposition~\ref{prop:apx-solve-cov}.
The running time analysis is identical to that of Theorem~\ref{thm:subg-mean}, we can move our guess $\nu$ at most $O(\log(d/\eps))$ times, and for each guess we invoke Proposition~\ref{prop:apx-solve-cov} to obtain a good primal or dual solution.
The overall running time is $\tilde O(N d \log(1/\tau) / \eps^6)$.

\begin{remark}
{\em
We note that in Algorithm~\ref{alg:sdp-boundedcov},
  we are interested in whether the optimal value of SDPs~\eqref{eqn:primal-sdp}~and~\eqref{eqn:dual-sdp} is at least $c_4$ or at most $0.9 c_4$.
Moreover, when we improve our guess $\nu$ using Lemma~\ref{lem:cov-good-dual}, the dual solution $M$ needs to be $0.95$-approximately optimal.
Therefore, we only need to solve SDPs~\eqref{eqn:primal-sdp}~and~\eqref{eqn:dual-sdp} to precision $(1-\eps')$ for some constant $\eps'$.
However, we do not know how to solve these SDPs directly in $O(Nd)$ time, and when we reduce them to packing/covering SDPs,
  the constraint on $w$ becomes the objective function in SDP~\eqref{eqn:primal-standard}, and we must solve SDP~\eqref{eqn:primal-standard} more precisely to precision $1 - \frac{\eps}{10}$ (see, e.g., Lemma~\ref{lem:primal-dual-convert-back}).
This is the only reason that we need to pay $\poly(1/\eps)$ in the running time of Algorithm~\ref{alg:sdp-boundedcov}.}
\end{remark}

\section{Conclusions and Future Directions} \label{sec:conc}

In this paper, we studied the problem of robust high-dimensional mean estimation for structured distribution families
in the presence of a constant fraction of corruptions. As our main technical contribution, 
we gave the first algorithms with dimension-independent error
guarantees for this problem that run in nearly-linear time. We hope that this work will serve as the starting point for the
design of faster algorithms for high-dimensional robust estimation. 

A number of natural directions suggest themselves:  Do our techniques generalize to robust covariance estimation? 
We believe so, but we have not explored this direction in the current work. 
Can we obtain nearly-linear time robust algorithms for other inference tasks 
under sparsity assumptions~\cite{BDLS17} (e.g., for robust sparse mean estimation or robust sparse PCA)? 
Can we speed-up the convex programs obtained via the SoS hierarchy in this setting~\cite{HopkinsL18, KothariSS18}?

The running time of our algorithms is $\tilde{O}(Nd)/\poly(\eps)$, i.e., it is 
nearly-linear when the fraction of corruptions $\eps$ is constant. Can we avoid the extraneous $\poly(1/\eps)$ dependence in the runtime?
We believe progress in this direction is attainable.
Note that solving a single covering SDP to multiplicative accuracy $(1+\eps)$ incurs a
$\poly(1/\eps)$ slowdown. Is it possible to reframe the underlying optimization problem so that 
a constant factor multiplicative accuracy suffices? 
Alternatively, is it possible to speed-up the iterative filtering technique of~\cite{DKKLMS16}?
Exploring alternate certificates of robustness may be a promising avenue towards these goals.

\section*{Acknowledgments}
We thank Alistair Stewart for useful discussions.

\bibliographystyle{alpha}
\bibliography{allrefs}

\appendix

\section{Setting Constants in Section~\ref{sec:mean}}
\label{apx:const}
In this section, we describe how to set universal constants $c_1, \ldots, c_7$ in Section~\ref{sec:mean}.
The constants are set in the following order: $c_1$, $c_2$, $c_4$, $c_5$, $c_7$, $c_6$, and $c_3$.
In this order, every $c_i$ only depends on the constants set before it, and there are only lower bounds on the value of $c_i$, so we can set $c_i$ to a sufficiently large constant.
Note that $c_3$ is the last constant we choose, and our guarantee at the end of the day is to output some hypothesis vector $\hat \mu$ that is close to the true mean $\mus$:
  $\normtwo{\hat \mu - \mus} \le c_3 \delta$.

Recall that in Section~\ref{sec:mean}, $0 < \eps < 1/3$, $\delta = c_1 \eps \sqrt{\ln(1/\eps)}$, $\delta_2 = c_1 \eps \ln(1/\eps)$, and $\beta = \sqrt{\eps \ln(1/\eps)}$.

The constant $c_1$ appears in the concentration bounds for the good samples (Condition~\eqref{eqn:good-sample-moments}),
and it is related to the constants in Chernoff bounds and Hanson-Wright inequality.
We can set $c_1$ to be any constant that Condition~\eqref{eqn:good-sample-moments} holds with the right sample complexity.

The constant $c_2$ is a threshold on $r = \normtwo{\nu - \mus}$.
When $r \ge c_2 \beta$, we can show that $\OPT$ is roughly $1 + r^2$.
We set $c_2$ to satisfy $\delta_2 + 2 \delta (c_2 \beta) \le 0.1 (c_2 \beta)^2$ as required by Lemma~\ref{lem:ub-lb-opt}.

The constant $c_4$ shows up in the branching statement of Algorithm~\ref{alg:pc-sdp}.
If $\OPT \le 1 + c_4 \beta^2$ we use the primal SDP solution, otherwise we use the dual SDP solution.
We set $c_4$ to satisfy $0.9 c_4^2 \ge 1.1 c_2^2$ in the proof of Lemma~\ref{lem:good-dual-better-nu}, and $\frac{c_4}{20}\beta^2 \ge \frac{\eps}{10}$ in the proof of Proposition~\ref{prop:approx-sdp}.

If we use the dual solution, we know that $r \ge c_2 \beta$.
If we use the primal solution, we have $r \le c_5 \beta$.
We choose $c_5$ where $c_5 \ge c_2$ and $0.9 c_5^2 \ge c_4$ as needed in the proof of Lemma~\ref{lem:wrong-mean-primal-nosol}.

In the proof of Lemma~\ref{lem:wrong-mean-primal-nosol}, the constants $c_6$ and $c_7$ appear when we argue that the bad samples contribute at least $\Omega(\delta^2/\eps)$ to the second-moment, and the good samples contribute at least $1 - O(\delta^2/\eps)$.
We choose $c_7$ such that $c_7 \ge 1 + \frac{2 c_5 \beta}{\sqrt{\ln(1/\eps)}}$, and $c_6$ such that $\frac{c_1^2 c_6^2}{2} \ge c_4 + c_1 c_7$.
Finally, because the good samples shift the mean by at most $\delta$, if the empirical mean is off by more than $c_3 \delta$ then most of the error are from the bad samples.
We choose $c_3$ so that $c_3 \ge c_6 + 1 + \frac{2c_5\sqrt{\eps}}{c_1}$.

\end{document}